\documentclass{article} 
\usepackage{iclr2024_conference,times}

\usepackage[utf8]{inputenc} % allow utf-8 input
\usepackage[T1]{fontenc}    % use 8-bit T1 fonts
\usepackage{hyperref}       % hyperlinks
\usepackage{url}            % simple URL typesetting
\usepackage{booktabs}       % professional-quality tables
\usepackage{amsfonts}       % blackboard math symbols
\usepackage{nicefrac}       % compact symbols for 1/2, etc.
\usepackage{microtype}      % microtypography
\usepackage{xcolor}         % colors

\usepackage{amsmath,amssymb,amsfonts}
\usepackage{amsthm}
\usepackage{mathtools}
\mathtoolsset{showonlyrefs}
\usepackage{bm}
\usepackage{subcaption}
\usepackage{algorithm}
\usepackage{algpseudocode}

\newtheorem{theorem}{Theorem}
\newtheorem{lemma}[theorem]{Lemma} 

\newtheorem{remark}[theorem]{Remark}

\newtheorem{corollary}[theorem]{Corollary}
\theoremstyle{definition}

\newcommand{\R}{\mathbb{R}}
\newcommand{\N}{\mathbb{N}}
\renewcommand{\d}{\mathrm{d}}
\renewcommand{\P}{\mathcal P_2}

\newcommand{\E}{\mathbb E}
\newcommand{\Sp}{\mathbb S}

\newcommand{\F}{\mathcal F}

\newcommand{\V}{\mathcal V}
\newcommand{\zb}{\bm}

\mathtoolsset{showonlyrefs}
\newcommand{\Beta}{\mathrm{B}}
\newcommand\dx{\mathrm{d}}

\title{Generative Sliced MMD Flows with Riesz \\ Kernels}

\author{
  Johannes Hertrich$^1$, Christian Wald$^2$, Fabian Altekr\"uger$^3$, Paul Hagemann$^2$ \\
  $^1$ University College London, 
  $^2$ Technische Universit\"at Berlin,
  $^3$ Humboldt-Universit\"at zu Berlin\\
  Correspondence to: \texttt{j.hertrich@ucl.ac.uk}
}

\iclrfinalcopy 

\begin{document}

\maketitle

\begin{abstract}
Maximum mean discrepancy (MMD) flows 
suffer from high computational costs in large scale computations.
In this paper, we show that MMD flows with Riesz kernels $K(x,y) = - \|x-y\|^r$, $r \in (0,2)$
have exceptional properties which allow their efficient computation.
We prove that the MMD of Riesz kernels, which is
also known as energy distance, coincides with the MMD of their sliced version.
As a consequence, the computation of gradients of MMDs can be performed in the one-dimensional setting.
Here, for $r=1$, a simple sorting algorithm can be applied to reduce the complexity
from $O(MN+N^2)$ to $O((M+N)\log(M+N))$ 
for two  measures with $M$ and $N$ support points.
As another interesting follow-up result, the MMD of compactly supported measures
can be estimated from above and below by the Wasserstein-1 distance.
For the implementations we approximate the gradient of the sliced MMD by using only a finite number $P$ of slices. 
We show that the resulting error has complexity \smash{$O(\sqrt{d/P})$}, where $d$ is the data dimension. 
These results enable us to train generative models by approximating MMD gradient flows by neural networks even
for image applications. We demonstrate the efficiency of our model by image generation on MNIST, FashionMNIST and CIFAR10. 
\end{abstract}

%-----------------------------------------------
\section{Introduction}
%-----------------------------------------------

With the rise of generative models, the field of gradient flows in measure spaces received increasing attention.
Based on classical Markov chain Monte Carlo methods, 
\citet{WT2011} proposed 
to apply the Langevin dynamics for inferring samples from a known probability density function.
This corresponds to simulating a Wasserstein gradient flow with respect to the Kullback-Leibler divergence, see \cite{JKO1998}.
Closely related to this approach are current state-of-the-art image generation methods like score-based models \citep{SE2019,SE2020} or diffusion models \citep{ho2020denoising,song2021scorebased}, which significantly outperform classical generative models like GANs \citep{goodfellow14} or VAEs \citep{KW2013}.
A general aim of such algorithms \citep{arbel2021annealed,ho2020denoising,Wu2020} is to establish a path between input and target distribution, where "unseen" data points are established via the randomness of the input distribution.
Several combinations of such Langevin-type Markov chain Monte Carlo methods with other generative models were proposed in \citep{ben-hamu22,hagemann2023generalized, Wu2020}.
Gradient flows on measure spaces with respect to other metrics are considered in \citep{LFS2022,DWYZ2023,GWJDZ2020,L2017,liu2016stein} under the name Stein variational gradient descent.

For approximating gradient flows with respect to other functionals than the KL divergence, the authors of \citep{altekruger2023neural,ansari2021refining,ASM2022,BPKC2022,Fan22,pmlr-v97-gao19b,garcia2023optimization,heng2023deep,MKLGSB2021,peyre2015entropic} proposed the use of suitable forward and backward discretizations. 
To reduce the computational effort of evaluating distance measures on high-dimensional probability distributions, the sliced Wasserstein metric was introduced in \citep{RPDB2012}.
The main idea of the sliced Wasserstein distance is to compare one-dimensional projections of the corresponding probability distributions instead of the distributions themselves.
This approach can be generalized to more general probability metrics \citep{kolouri22} and was applied in the context of Wasserstein gradient flows in \citep{bonet2022efficient,liutkus19}.

For many generative gradient-flow methods it is required that the considered functional can be evaluated based on samples.
For divergence-based functionals like the Kullback-Leibler or the Jensen-Shannon divergence, a variational formulation leading to a GAN-like evaluation procedure is provided in \citep{Fan22}.
In contrast, the authors of \citep{altekruger2023neural,arbel19,GAG2021} use functionals based on the maximum mean discrepancy (MMD) which can be directly evaluated based on empirical measures.
For positive definite kernels, it can be shown under some additional assumptions that MMD defines a metric on the space of probability distributions, see e.g., \citep{GBRSS2012,SFL2011,SGFSL2010}.
If the considered kernel is smooth, then 
\citet{arbel19} proved that Wasserstein gradient flows can be fully described by particles.
Even though this is no longer true for non-smooth kernels \citep{hertrich2023wasserstein}, 
\cite{altekruger2023neural} pointed out that particle flows are Wasserstein gradient flows at least with respect to a restricted functional.
In particular, we can expect that particle flows provide an accurate approximation of Wasserstein gradient flows as long as the number of particles is large enough.

\begin{figure}[t]
\centering
\begin{minipage}{0.32\textwidth}
\includegraphics[width=0.82\textwidth]{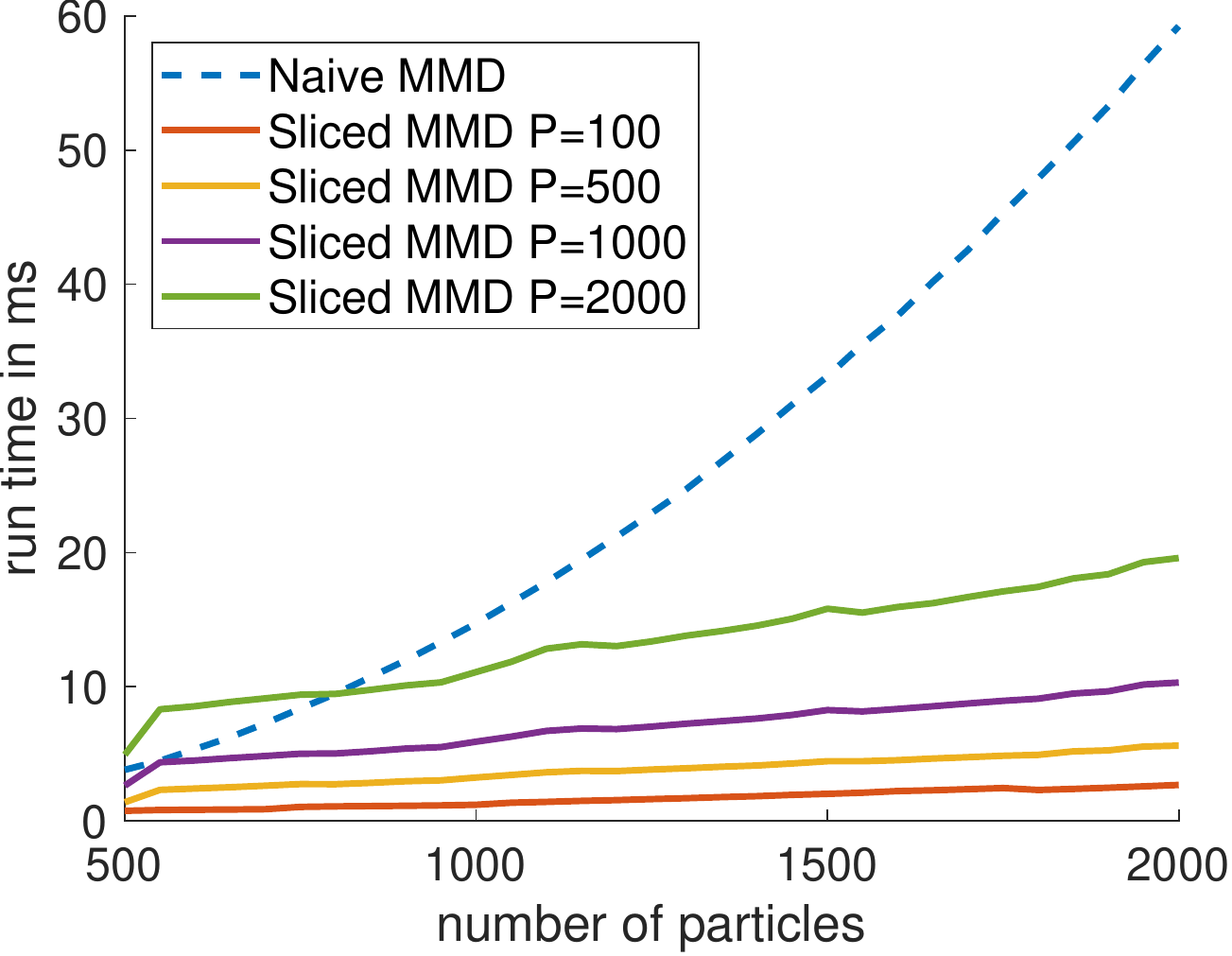}
\end{minipage}
\begin{minipage}{0.32\textwidth}
\includegraphics[width=0.85\textwidth]{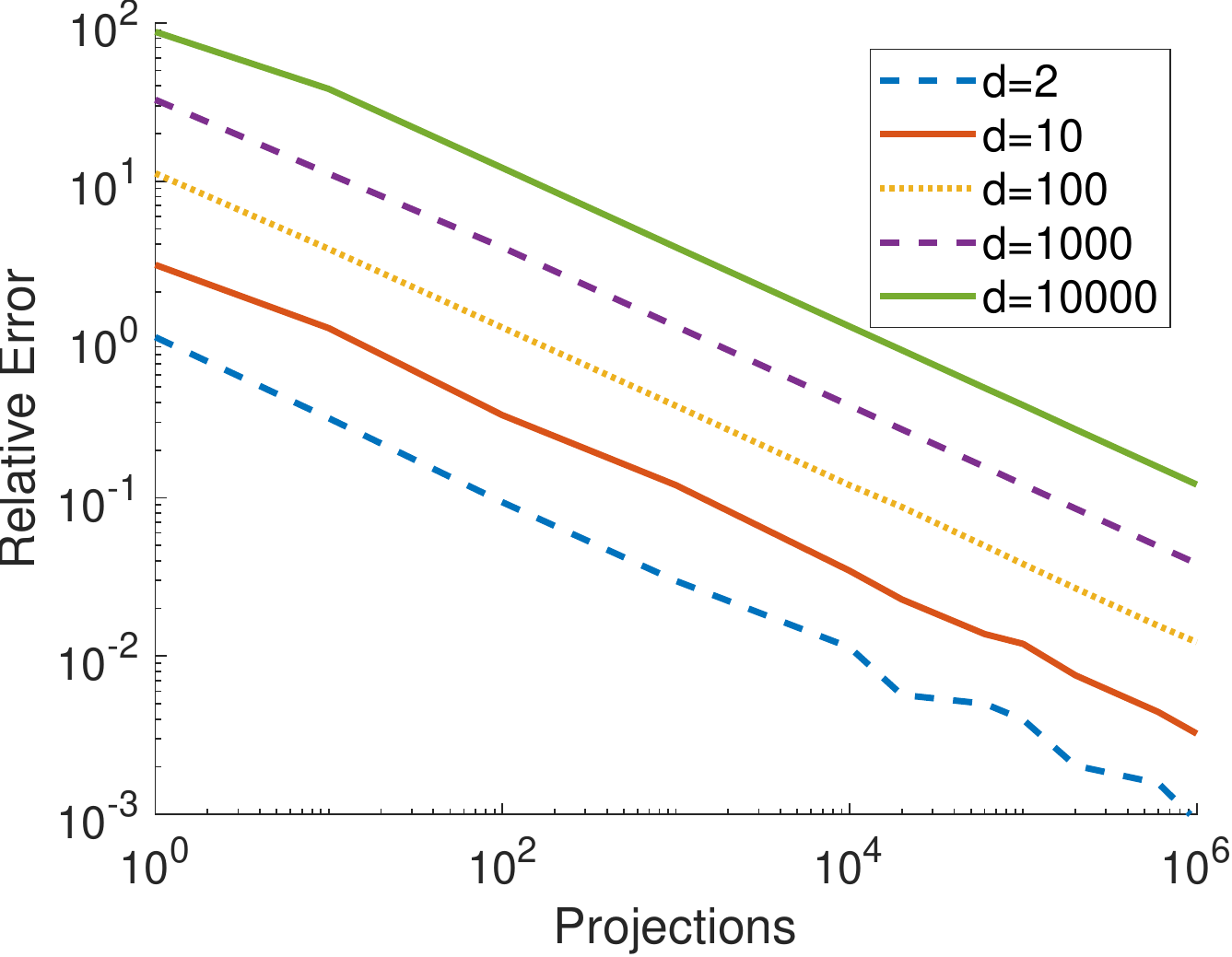}
\end{minipage}
\begin{minipage}{0.32\textwidth}
\includegraphics[width=0.85\textwidth]{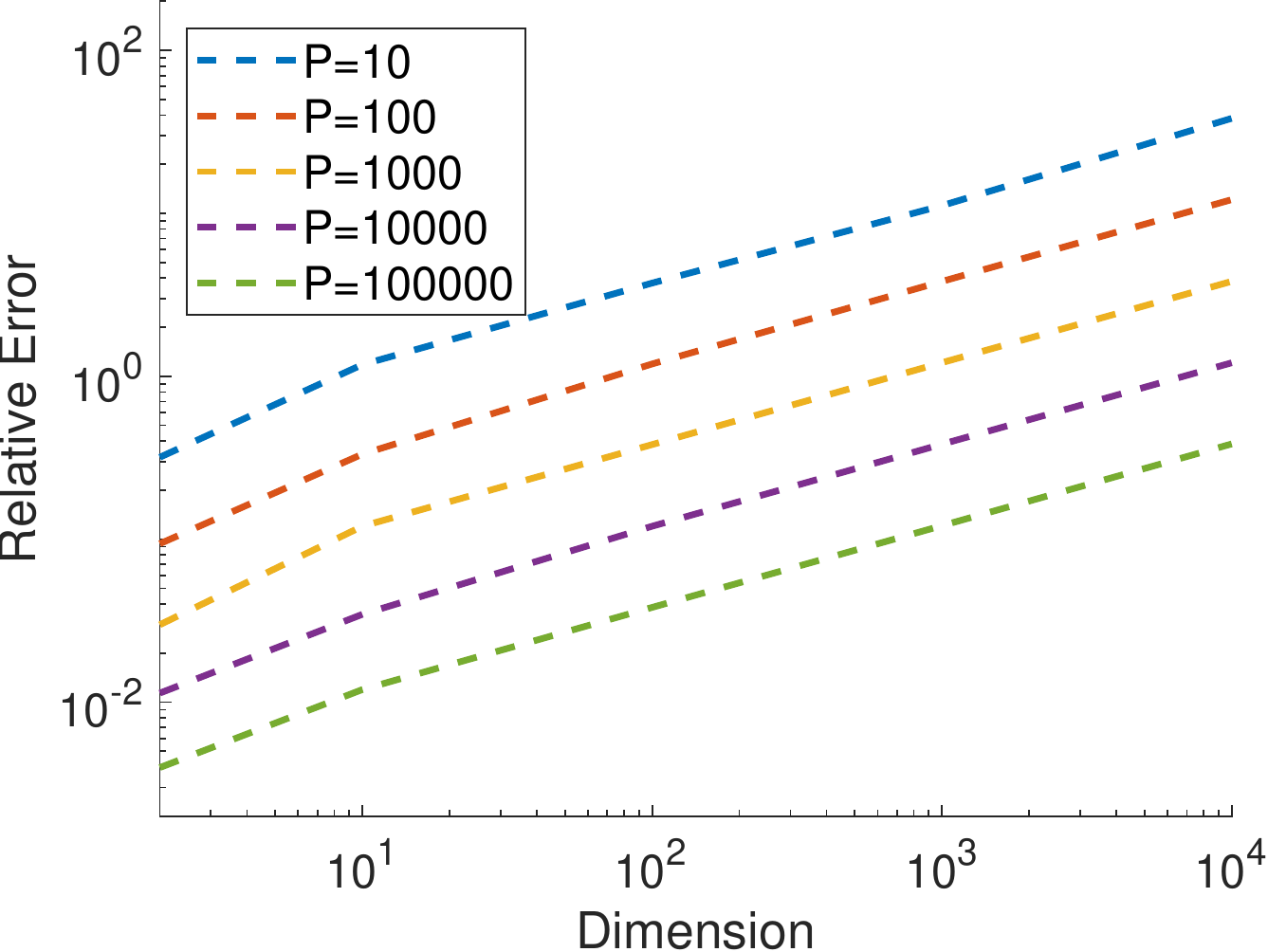}
\end{minipage}
\caption{Left: Comparison of run time for $1000$ gradient evaluations of naive MMD and sliced MMD with different number of projections $P$ in the case $d=100$. Middle and right: Relative error of the gradients of sliced MMD and MMD with respect to the number $P$ of projections and the dimension $d$. The results show the relative error behaves asymptotically as $O(\sqrt{d/P})$ as shown in Theorem~\ref{thm:convergence_rate}.}
    \label{fig:comparison_runtime}
\end{figure}

\textbf{Contributions.}
The computational complexity of MMD between two empirical measures with $N$ and $M$ support points depends quadratically on $N$ and $M$, which makes large scale computations impossible.
In this paper, we focus on the MMD with \emph{Riesz kernels}
\begin{equation}\label{eq:riesz}
K(x,y) = - \|x-y\|^r,\quad r \in (0,2),
\end{equation}
also known as energy distance \citep{mmd_energy_eq,szekely2002,szekely_energy}.
We show that Riesz kernels have the outstanding property that their MMD coincides with the sliced MMD of univariate Riesz kernels. It is this property that enables us to reduce the computation of (gradients of) MMD to the one-dimensional setting.
In the case of $r=1$, we propose a simple and computationally very efficient sorting algorithm for computing the gradient of the one-dimensional MMD  with complexity $O((M+N)\log(M+N))$. 
Considering that our numerical examples use between 10.000 and 50.000 particles, this leads to an incredible speed-up for gradient computations of (sliced) MMD as illustrated in the left plot of Figure~\ref{fig:comparison_runtime}.
Our approach opens the door to applications in image processing, where we have usually to cope with high dimensional data. 

In practice, sliced probability metrics are evaluated by replacing the expectation over all projections by the empirical expectation resulting in a finite sum. 
In the case of sliced MMD with Riesz kernels and $r=1$, we prove that the error induced by this approximation behaves asymptotically as $O(\sqrt{d/P})$, where $d$ is the data dimension and $P$ the number of projections, see the middle plot in Figure~\ref{fig:comparison_runtime} for an illustration.
The square root scaling of the error in the dimension $d$ ensures that an accurate computation of the sliced MMD with Riesz kernels is possible even in very high dimensions as demonstrated in the right plot in Figure~\ref{fig:comparison_runtime}. Taking into account the number of projections, the overall complexity of the computation of the derivatives of MMD is $O(dP(M+N)\log(M+N))$.

We apply the cheap evaluation of MMD gradients to compute MMD particle flows starting with samples from an initial probability measure $\mu_0$ to samples from a predefined target distribution $\nu$, which is given by samples.
Finally, we derive a generative model by training a sequence $(\Phi_l)_{l=1}^L$ of neural networks, where each $\Phi_l$ approximates a certain number of steps of the particle flow.
This allows us to train our network iteratively. 
In particular, during the training and evaluation procedure, we always consider only one of the networks $\Phi_l$ at the same time.
This allows an efficient training with relatively low resources even though all networks $\Phi_l$ together have a large number of parameters.
We demonstrate the efficiency of our generative sliced MMD flows for image generation on MNIST, FashionMNIST and CIFAR10.

\textbf{Related Work.}
Gradient flows with respect to MMD functionals are considered in \citep{altekruger2023neural,arbel19,hertrich2023wasserstein,kolouri22}. 
However, due to the quadratic complexity of the computation of the derivative of MMD functionals in the number of support points of the involved measures, these papers have a rather theoretical scope and applications are limited to measures supported on a few hundred points.
In order to reduce the dimension of the problem, \citet{kolouri22} consider a sliced version of MMD. 
This is motivated by the success of sliced Wasserstein distances \citep{RPDB2012}, which were used for deriving gradient flows in \citep{bonet2022efficient,liutkus19,nguyen2023hierarchical,nguyen2021distributional}.
In particular, \citet{kolouri22} observe that the sliced MMD is again a MMD functional \emph{with a different kernel}. 
We use this result in Section~\ref{sec:2}.
\citet{VG2023} bound Wasserstein distances and MMD against each other. However, they use strong assumptions on the kernel,
which are not fulfilled for the negative distance kernel.
In very low dimensions, fast evaluations of MMD and their gradients were proposed in \citep{GPS2012,TSGSW2011} based on fast Fourier summation using the non-equispaced fast Fourier transforms (NFFT), see \citep[Sec.~7]{PPST2018} and references therein.
Unfortunately, since the complexity of the NFFT depends exponentially on the data-dimension, these approaches are limited to applications in dimension four or smaller.
In a one-dimensional setting, the energy distance is related to the Cramer distance, see \citep{szekely2002}. In the context of reinforcement learning, \citet{lheritier2021cram} developed fast evaluation algorithms for the latter based on the calculation of cumulative distribution functions.

Finally, the authors of \citep{binkowski2018demystifying, DRG2015, li2017mmd,LSZ2015} apply MMD for generative modelling by constructing so-called MMD-GANs. However, this is conceptionally a very different approach since in MMD-GANs the discriminator in the classical GAN framework \citep{goodfellow14} is replaced by a MMD distance with a variable kernel. Also relevant is the direction of Sobolev-GANs \citep{mroueh2018sobolev} in which the discriminator is optimized in a Sobolev space, which is related to the RKHS of the Riesz kernel. Similar to GAN ideas this results in a max-min problem which is solved in an alternating fashion and is not related to gradient flows.

\textbf{Outline of the Paper.}
In Section~\ref{sec:2}, we prove that the sliced MMD with the one-dimensional Riesz kernel coincides with MMD of the scaled $d$-dimensional kernel. This can be used to establish an interesting lower bound on the MMD
by the Wasserstein-1 distance.
Then, in Section~\ref{sec3} we propose a sorting algorithm for computing the derivative of the sliced MMD in an efficient way. 
We apply the fast evaluation of MMD gradients to simulate MMD flows and to derive a generative model in Section~\ref{sec:4} .
Section~\ref{sec:5} shows numerical experiments on image generation. Finally, conlusions are drawn in Section~\ref{sec:6}. The appendices contain the proofs and  supplementary material.

%--------------------------------------------------------
\section{Sliced MMD for Riesz Kernels} \label{sec:2}
%--------------------------------------------------------
Let $\mathcal P(\R^d)$ denote the set of probability measures on $\mathbb R^d$ and
$\mathcal P_p(\R^d)$ its subset of measures
with finite $p$-th moment, i.e., $\int_{\R^d}\|x\|^p \d \mu(x)< \infty$.
Here $\|\cdot\|$ denotes the Euclidean norm on $\R^d$.
For a symmetric, positive definite kernel $K\colon\R^d\times\R^d\to\R$, 
the \emph{maximum mean discrepancy} (MMD) $\mathcal D_K \colon \mathcal P(\R^d) \times \mathcal P(\R^d) \to \mathbb R$
is the square root of 
$
\mathcal D_K^2(\mu,\nu)\coloneqq \mathcal E_K(\mu-\nu),
$
where $\mathcal E_K$ is the \emph{interaction energy}  of signed measures on $\mathbb R^d$ defined by
\begin{equation*}
\mathcal E_K(\eta) \coloneqq  \frac12 \int_{\R^d}\int_{\R^d} K(x,y)  \,\d \eta(x)\d \eta(y).
\end{equation*}
Due to its favorable properties, see Appendix~\ref{app:motivation_riesz}, we are interested in Riesz kernels
\begin{equation}
K(x,y) = - \|x-y\|^r,\quad r \in (0,2).
\end{equation}
These kernels are only conditionally positive definite, 
but can be  extended to positive definite kernels by $\tilde K (x,y) = K(x,y) - K(x,0) - K(0,y)$, see also Remark \ref{rem:riesz_BM}.
Then it holds for $\mu,\nu \in \mathcal P_r(\R^d)$ that
$\mathcal D_K(\mu,\nu) = \mathcal D_{\tilde K}(\mu,\nu)$, see \citep[Lemma 3.3 iii)]{NS2021}. 
Moreover, for Riesz kernels, $\mathcal D_K$ is a metric on $\mathcal P_r(\mathbb R^d)$, which is also known as
so-called energy distance \citep{mmd_energy_eq,szekely_energy}.
Note that we exclude the case $r=2$, since $\mathcal D_K$ is no longer a metric in this case.

However, computing MMDs on high dimensional spaces is computationally costly. 
Therefore, the \emph{sliced MMD} $\mathcal{SD}_{\mathrm k}^2: \mathcal P_2(\R^d) \times \mathcal P_2(\R^d) \to \mathbb R$
was considered in the literature, see e.g., \cite{kolouri22}.
For a symmetric 1D kernel ${\mathrm k}\colon\R\times\R\to\R$ it is given by
\begin{align*}
\mathcal{SD}_{\mathrm k}^2(\mu,\nu)
\coloneqq 
\E_{\xi \sim \mathcal U_{\Sp^{d-1}}}[\mathcal D_{\mathrm k}^2({P_\xi}_\#\mu,{P_\xi}_\#\nu)]
\end{align*}
with the push-forward measure ${P_\xi}_\# \mu \coloneqq \mu \circ P_\xi^{-1}$
of the projection 
$P_\xi(x) \coloneqq \langle \xi,x\rangle$
and the uniform distribution \smash{$\mathcal U_{\Sp^{d-1}}$} on the sphere \smash{$\mathbb S^{d-1}$}.
By interchanging the integrals from the expectation and the definition of MMD, 
\citet{kolouri22} observed
that the sliced MMD is equal to the MMD 
with an associate kernel $\mathrm{K}\colon\R^d\times\R^d\to\R$. 
More precisely, it holds
\begin{equation} \label{appr_kernel}
\mathcal{SD}_{\mathrm k}^2(\mu,\nu)
=
\mathcal D_{\mathrm{K}}^2(\mu,\nu),
\quad  \text{with} \quad \mathrm{K}(x,y)\coloneqq \E_{\xi\sim \mathcal U_{\Sp^{d-1}}}[{\mathrm k}(P_\xi(x),P_\xi(y))].
\end{equation}
By the following theorem, this relation becomes more simple when dealing with Riesz kernels, 
since in this case the associate kernel is a Riesz kernel as well.

\begin{theorem}[Sliced Riesz Kernels are Riesz Kernels]\label{sliced:unsliced}
Let ${\mathrm k}(x,y) \coloneqq -|x-y|^r$, $r\in(0,2)$.
Then, for $\mu, \nu \in \mathcal{P}_r(\R^d)$, it holds 
$
\mathcal{SD}_{\mathrm k}^2(\mu,\nu)
=
\mathcal D_{\mathrm{K}}^2(\mu,\nu)
$
with the associated scaled Riesz kernel 
\begin{align*}
 {\mathrm K} (x,y)  \coloneqq -c_{d,r}^{-1}\|x-y\|^r, \quad 
c_{d,r} 
\coloneqq \frac{\sqrt{\pi}\Gamma(\frac{d+r}{2})}{\Gamma(\frac{d}{2})\Gamma(\frac{r+1}{2})}.
\end{align*}
\end{theorem}

The proof is given in Appendix~\ref{proof:sliced:unsliced}.
The constant $c_{d,r}$ depends asymptotically with $O(d^{r/2})$ on the dimension. In particular, it should be ``harder'' to estimate the MMD or its gradients in higher dimensions via slicing. We will discuss this issue more formally later in Remark~\ref{rem:complexity_gradient_evaluation}.
For $r=1$, we just write $c_d\coloneqq c_{d,1}$ and can consider measures in $\mathcal P_1(\R^d)$.
Interestingly, based on Theorem \ref{sliced:unsliced}, 
we can establish a relation between the MMD and
the Wasserstein-1 distance on $\mathcal P_1 (\R^d)$ defined by
\[
\mathcal W_1(\mu,\nu) \coloneqq \min_{\pi \in \Pi(\mu,\nu)} \int \|x-y\| \, \d \pi(x,y),
\]
where $\Pi(\mu,\nu)$ denotes the set of measures in $\mathcal P_1(\R^d \times \R^d)$ with marginals $\mu$ and $\nu$. 
This also shows that Conjecture 1 in \citep{MD2023} can only hold for
non-compactly supported measures.
The proof is given in Appendix \ref{proof:rel}.

\begin{theorem}[Relation between $\mathcal D_K$ and $\mathcal W_1$ for Distance Kernels]\label{thm:rel}
Let $K(x,y) \coloneqq - \|x-y\|$.
Then, it holds for $\mu, \nu \in \mathcal P_1(\mathbb{R}^d)$ that
$
    2 \mathcal{D}_{K}^2(\mu, \nu) 
		\leq  
		 \mathcal{W}_{1}(\mu, \nu). 
$
If $\mu$ and $\nu$ are additionally supported on the ball $B_R(0)$, then
there exists a constant $C_d>0$ such that
$
    		 \mathcal{W}_{1}(\mu, \nu) 
		\leq  
		C_d R^{\frac{2d+1}{2d+2}} 
		\mathcal{D}_{K}(\mu, \nu)^{\frac{1}{d+1}} .
$
\end{theorem}

The fact that the sample complexities of MMD and Wasserstein-1 are $O(n^{-1/2})$ \cite[Chapter 4.1]{GBRSS2012} and $O(n^{-1/d})$ \cite[Chapter 8.4.1]{peyre2020computational} suggests, that the exponent of $\mathcal D_K$ in Theorem~\ref{thm:rel} cannot be improved over $1/d$.

%-------------------------------------------------
\section{Gradients of Sliced MMD} \label{sec3}
%-------------------------------------------------
Next, we consider the functional
$\F_\nu\colon\P(\R^d)\to\R$ given by
\begin{equation}\label{eq:MMD_ojective_fun}
\F_\nu (\mu) \coloneqq \mathcal E_K(\mu) +\mathcal V_{K,\nu}(\mu) 
= \mathcal D_K^2(\mu,\nu) +\text{const}_{\nu}, 
\end{equation}
where $\V_{K, \nu}(\mu)$ is the so-called \emph{potential energy} 
\begin{align}\label{eq:potential}
    \V_{K, \nu}(\mu) \coloneqq - \int_{\R^d} \int_{\R^d} K(x,y) \,  \d\nu(y) \, \d \mu(x)    
\end{align}
acting  as an attraction term between the masses of $\mu$ and $\nu$, while the 
interaction energy $\mathcal E_K$ is a repulsion term enforcing a proper spread of $\mu$.
For the rest of the paper, we always consider the negative distance kernel $K(x,y):=-\|x-y\|$, which is the Riesz kernel \eqref{eq:riesz} with $r=1$.
Then, we obtain directly from the metric property of MMD that the minimizer of the non-convex functional $\F_\nu$ is given by $\nu$.
We are interested in computing gradient flows of $\F_\nu$ towards this minimizer.
However, the computation of gradients in measure spaces for discrepancy functionals with non-smooth kernels is highly non-trivial and computationally costly, see e.g., \citep{altekruger2023neural,CDEFS2020,hertrich2023wasserstein}.

As a remedy, we focus on a discrete form of the $d$-dimensional MMD.
More precisely, we assume that $\mu$ and $\nu$ are empirical measures, i.e., they are of the form
$\mu=\frac{1}{N}\sum_{i=1}^N\delta_{x_i}$ 
and 
$\nu= \frac{1}{M}\sum_{j=1}^M\delta_{y_j}$  
for some $x_j,y_j \in \R^d$.
Let
$\zb {x} \coloneqq (x_1,\ldots,x_N)$ and 
$\zb y \coloneqq (y_1,\ldots,y_M)$.
Then the functional $\F_\nu$ reduces to the function 
$F_d(\cdot|\zb y)\colon\R^{Nd}\to\R$ given by
\begin{align}
F_d(\zb x|\zb y)&=-\frac{1}{2 N^2}\sum_{i=1}^N\sum_{j=1}^N\|x_i-x_j\|+\frac{1}{MN}\sum_{i=1}^N\sum_{j=1}^M \|x_i-y_j\|\label{eq:discrete_MMD}\\
&=\mathcal D_{K}^2\Big(\frac1N\sum_{i=1}^N\delta_{x_i},\frac1M\sum_{j=1}^M \delta_{y_j}\Big)+\mathrm{const}_{\zb y}.
\end{align}
In order to evaluate the gradient of $F_d$ with respect to the support points $\zb x$, we use Theorem~\ref{sliced:unsliced} to rewrite $F_d(\zb x|\zb y)$ as
\begin{align}\label{eq:discrete_MMD_1D}
F_d(\zb x|\zb y)=c_d\E_{\xi \sim\mathcal U_{\Sp^{d-1}}}[F_1(\langle \xi,x_1\rangle,...,\langle\xi,x_N\rangle|\langle \xi,y_1\rangle,...,\langle \xi,y_M\rangle)].
\end{align}
Then, the gradient of $F_d$ with respect to $x_i$ is given by
\begin{equation}\label{eq:derivative_MMD}
\nabla_{x_i}F_d(\zb x|\zb y)=c_d\E_{\xi \sim \mathcal U_{\Sp^{d-1}}}[\partial_i F_1(\langle \xi,x_1\rangle,...,\langle\xi,x_N\rangle|\langle \xi,y_1\rangle,...,\langle \xi,y_M\rangle)\xi],
\end{equation}
where $\partial_i F_1$ denotes the derivative of $F_1$ with respect to the $i$-th component of the input.
Consequently, it suffices to compute gradients of $F_1$ in order to evaluate the gradient of $F_d$.

\textbf{A Sorting Algorithm for the 1D-Case.}
Next, we derive a sorting algorithm to compute the gradient of $F_1$ efficiently. 
In particular, the proposed algorithm has complexity $O((M+N)\log(M+N))$ even though the definition of $F_1$ in \eqref{eq:discrete_MMD} involves $N^2+MN$ summands.

To this end, we split the functional $F_1$ into interaction and potential energy, i.e., $F_1(\zb x|\zb y)=E(\zb x)+V(\zb x|\zb y)$ with
\begin{equation}\label{eq:discrete_interaction_potential}
E(\zb x):=-\frac{1}{2 N^2}\sum_{i=1}^N\sum_{j=1}^N|x_i-x_j|,\quad 
V(\zb x|\zb y):=\frac{1}{NM}\sum_{i=1}^N\sum_{j=1}^M |x_i-y_j|.
\end{equation}
Then, we can compute the derivatives of $E$ and $V$ by the following theorem  which proof is given in Appendix~\ref{proof:sorting}.
%------------------------------
\begin{theorem}[Derivatives of Interaction and Potential Energy]\label{thm:sorting}
Let $x_1,...,x_N\in\R$ be pairwise disjoint 
and $y_1,...,y_M\in\R$ such that $x_i\neq y_j$ for all $i=1,...,N$ and $j=1,...,M$.
Then, $E$ and $V$ are differentiable with
\[
\nabla_{x_i} E(\zb x)= \frac{N + 1 -2\sigma^{-1}(i)}{N^2},\quad \nabla_{x_i} V(\zb x|\zb y)=\frac{2\,\#\{j\in\{1,...,M\}:y_j<x_i\}-M}{MN},
\]
where $\sigma\colon\{1,...,N\}\to \{1,...,N\}$ is the permutation with $x_{\sigma(1)}< ... < x_{\sigma(N)}$.
\end{theorem}
%---------------------------------
Since $V$ is convex, we can show with the same proof that 
$$
\frac{2\,\#\{j\in\{1,...,M\}:y_j<x_i\}-M}{MN}\in \partial_{x_i} V(\zb x|\zb y),
$$
where $\partial_{x_i}$ is the subdifferential ov $V$ with respect to $x_i$ whenever $x_i=y_j$ for some $i,j$.
By Theorem \ref{thm:sorting}, we obtain that 
$\nabla F_1(\zb x|\zb y) = \nabla E(\zb x)+\nabla V(\zb x|\zb v)$ 
can be computed by
Algorithm~\ref{alg:interaction_sort} and Algorithm~\ref{alg:potential_sort}
with
complexity $O(N\log(N))$ and  $O((M+N)\log(M+N))$, respectively.
The complexity is dominated by the sorting procedure.
Both algorithms can be implemented in a vectorized form for computational efficiency.
Note that by Lemma~\ref{l2+D} from the appendix, the discrepancy with Riesz kernel and $r=1$ can be represented by the cumulative distribution functions (cdfs) of the involved measures. Since the cdf of an one-dimensional empirical measure can be computed via sorting, we also obtain an $O((N+M)\log(M+N))$ algorithm for computing the one-dimensional MMD itself and not only for its derivative.

\begin{figure}[t]
\begin{algorithm}[H]
\begin{algorithmic}
\State \textbf{Input:} $x_1,...,x_N\in \R$ with $x_i\neq x_j$ for $i\neq j$.
\State \textbf{Algorithm:}
\State Compute $\sigma_1,...\sigma_N=\mathrm{argsort}(x_1,...,x_N)$.
\State Compute $g_i=-\frac{2\sigma_i^{-1}-1-N}{N^2}$.
\State \textbf{Output:} $(g_1,...,g_N)=\nabla E(x_1,...,x_N)$.
\end{algorithmic}
\caption{Derivative of the interaction energy $E$ from \eqref{eq:discrete_interaction_potential}.}
\label{alg:interaction_sort}
\end{algorithm}
\begin{algorithm}[H]
\begin{algorithmic}
\State \textbf{Input:} $x_1,...,x_N\in\R$, $y_1,...,y_M\in\R$ with $x_i\neq y_j$.
\State \textbf{Algorithm:}
\State Compute $\sigma_1,...,\sigma_{N+M}=\mathrm{argsort}(x_1,...,x_N,y_1,...,y_M)$
\State Initialize $\tilde h_1=\cdots=\tilde h_{M+N}=0$.
\For{$j=1,...,M$}
\State Set $\tilde h_{\sigma(N+j)}=1$.
\EndFor
\State Set $h=2\,\mathrm{cumsum}(\tilde h)-1$
\For{$i=1,...,N$}
\State Set $g_i=\frac{h_{\sigma^{-1}(i)}}{MN}$,
\EndFor
\State \textbf{Output:} $(g_1,...,g_N)=\nabla V(x_1,...,x_N|y_1,\ldots,y_M)$.
\end{algorithmic}
\caption{Derivative of the potential energy $V$ from \eqref{eq:discrete_interaction_potential}.}
\label{alg:potential_sort}
\end{algorithm}
\end{figure}

\textbf{Stochastic Approximation of Sliced MMD Gradients for $r=1$.}
To evaluate the gradient of $F_d$ efficiently, we use a stochastic gradient estimator.
For $x_1,...,x_N,y_1,...,y_M\in\R^d$, we define for $P\in\N$ the stochastic gradient estimator of \eqref{eq:derivative_MMD}
as the random variable 
\begin{equation}\label{eq:def_grad_est1}
\tilde\nabla_P F_d(\zb x|\zb y)= \left(\tilde\nabla_{P,x_i}F_d(\zb x|\zb y) \right)_{i=1}^N
\end{equation}
where
\begin{equation}\label{eq:def_grad_est2}
\tilde \nabla_{P,x_i}F_d(\zb x|\zb y)
\coloneqq
\frac{c_d}{P}\sum_{p=1}^P\partial_i F_1(\langle \xi_p,x_1\rangle,...,\langle\xi_p,x_N\rangle|\langle \xi_p,y_1\rangle,...,\langle \xi_p,y_M\rangle)\xi_p,
\end{equation}
for independent random variables $\xi_1,...,\xi_P\sim\mathcal U_{\Sp^{d-1}}$.
We obtain  by \eqref{eq:derivative_MMD} that $\tilde \nabla F_d$ is unbiased, i.e.,
$
\E[\tilde\nabla_P F_d(\zb x|\zb y)]=\nabla F_d(\zb x|\zb y).
$
Moreover, the following theorem shows that the error of $\tilde \nabla_P F_d$ converges to zero for a growing number $P$ of projections. The proof uses classical concentration inequalities and follows directly from Corollary~\ref{cor:concentration} in Appendix~\ref{proof:concentration}.
%---------------------------
\begin{theorem}[Error Bound for Stochastic MMD Gradients]\label{thm:convergence_rate}
Let $x_1,...,x_N,y_1,...,y_M\in \R^d$. Then, it holds 
\[
\E[\|\tilde\nabla_P F_d(\zb x|\zb y)-\nabla F_d(\zb x|\zb y)\|]\in O\Big(\sqrt{\frac{d}{P}}\Big).
\]
\end{theorem}
%-----------------------------

To verify this convergence rate numerically, we draw $N=1000$ samples $x_1,...,x_N$ from a Gaussian mixture model with two components and $M=1000$ samples $y_1,...,y_M$ from a Gaussian mixture model with ten components.
The means are chosen randomly following a uniform distribution in $[-1,1]^d$ and the standard deviation is set to $0.01$.
Then, we compute numerically the expected relative approximation error between 
$\tilde \nabla_P F_d$ and $\nabla F_d$ for different choices of $P$ and  $d$.
The results are illustrated in the middle and in the right plot of Figure~\ref{fig:comparison_runtime}.
We observe that this numerical evaluation underlines the convergence rate of $O\Big(\sqrt{\frac{d}{P}}\Big)$. In particular, the error scales with $O(\sqrt{d/P})$, which makes the approach applicable for high-dimensional problems.

%----------------------------------------
\begin{remark}[Computational Complexity of Gradient Evaluations]\label{rem:complexity_gradient_evaluation}
The appearance of the $\sqrt{d}$ in the error bound is due to the scale factor $c_d$ between the MMD and the sliced MMD, which can be seen in the proof of Theorem~\ref{thm:convergence_rate}.
In particular,
we require $O(d)$ projections in order to approximate $\nabla F_d(\zb x|\zb y)$ by $\tilde \nabla_P F_d(\zb x|\zb y)$ up to a fixed expected error of $\epsilon$. Together with the computational complexity of
$
O(dP(N+M)\log(N+M))
$
for 
$
\tilde \nabla_P F_d(\zb x,\zb y)
$,
we obtain an overall complexity of
$O(d^2(N+M)\log(N+M))$ in order to approximate $\nabla F_d(\zb x|\zb y)$ up to an expected error of $\epsilon$.
On the other hand, the naive computation of (gradients of) $F_d(\zb x|\zb y)$ has a complexity of
$O(d(N^2+MN))$.
Consequently, we improve the quadratic complexity in the number of samples to $O(N\log(N))$. Here, we pay the price of quadratic instead of linear complexity in the dimension.
\end{remark}

\section{Generative MMD Flows}\label{sec:4}

In this section, we use MMD flows
with the negative distance kernel for generative modelling.
Throughout this section, we assume that we are given independent samples $y_1,...,y_M\in \R^d$ from a target measure
$\nu\in\P(\R^d)$ and define the empirical version of $\nu$ by $\nu_M\coloneqq \frac1M\sum_{i=1}^M\delta_{y_i}$.

\subsection{MMD Particle Flows}

In order to derive a generative model approximating $\nu$, we simulate a gradient flow of the functional $\F_\nu$ from \eqref{eq:MMD_ojective_fun}.
Unfortunately, the computation of gradient flows in measure spaces for $\F_\nu$ is highly non-trivial 
and computationally costly, see \citep{altekruger2023neural,hertrich2023wasserstein}.
Therefore, we consider the (rescaled) gradient flow with respect to the functional $F_d$ instead.
More precisely, we simulate for $F_d$ from \eqref{eq:discrete_MMD},
the (Euclidean) differential equation 
\begin{equation}\label{eq:ODE}
\dot {\zb x}=-N\,\nabla F_d(\zb x|\zb y),\quad x(0)=(x_1^{(0)},...,x_N^{(0)}),
\end{equation}
where the initial points $x_i^{(0)}$ are drawn independently from some measure
$\mu_0\in\mathcal P_2(\R^d)$. In our numerical experiments, we set $\mu_0$ to the uniform distribution on $[0,1]^d$. 
Then, for any solution $x(t)=(x_1(t),...,x_N(t))$ of \eqref{eq:ODE}, it is proven in \citep[Proposition 14]{altekruger2023neural} that the curve $\gamma_{M,N}\colon(0,\infty)\to\P(\R^d)$ defined by $\gamma_{M,N}(t)=\frac1N\sum_{i=1}^N\delta_{x_i(t)}$ is a Wasserstein gradient flow with respect to the function
\begin{align*}
\mathcal F\colon\P(\R^d)\to\R\cup\{\infty\}, \quad \mu\mapsto 
\begin{cases}
\F_{\nu_M},&$if $\mu=\frac{1}{N}\sum_{i=1}^N\delta_{x_i}$ for some $x_i\neq x_j\in\R^d,\\
+\infty,&$otherwise.$ 
\end{cases}
\end{align*}
Hence, we can expect for $M,N\to\infty$, that the curve $\gamma_{M,N}$ approximates the Wasserstein gradient flow with respect to $\F_\nu$.
Consequently, we can derive a generative model by simulating the gradient flow \eqref{eq:ODE}. To this end, we use the explicit Euler scheme
\begin{align}\label{eq:MMD_GD}
\zb x^{(k+1)}&=\zb x^{(k)}-\tau N \nabla F_d(\zb x^{(k)}|\zb y),
\end{align}
where $\zb x^{(k)} = (x_1^{(k)},...,x_N^{(k)})$ and $\tau>0$ is some step size.
Here, the gradient on the right-hand side can be evaluated very efficiently by the stochastic gradient estimator from \eqref{eq:def_grad_est1}.

\textbf{Momentum MMD Flows.}
To reduce the required number of steps in \eqref{eq:MMD_GD}, we introduce a momentum parameter.
More precisely, for some given momentum parameter $m\in[0,1)$ we consider the momentum MMD flow defined by the following iteration
\begin{equation}\label{eq:Mom_MMD_GD}
\begin{aligned}
\zb v^{(k+1)}&=\nabla F_d(\zb x^{(k)}|\zb y)
+m \, \zb v^{(k)} \\
\zb x^{(k+1)}&=\zb x^{(k)}-\tau N\ \zb v^{(k+1)},
\end{aligned}
\end{equation}
where $\tau>0$ is some step size,  $x_i^{(0)}$ are independent samples from a initial measure $\mu_0$ and $v_i^{(0)}=0$.
Note that the MMD flow \eqref{eq:MMD_GD} is a special case of the momentum MMD flow \eqref{eq:Mom_MMD_GD} with $m=0$. 

In Figure~\ref{fig:trajectories_momentum}, we illustrate the momentum MMD flow \eqref{eq:Mom_MMD_GD} and MMD flow \eqref{eq:MMD_GD} without momentum from a uniform distribution on $[0,1]^d$ to MNIST \citep{LBBH1998} and CIFAR10 \citep{K2009}. The momentum is set to $m=0.9$ for MNIST and to $m=0.6$ for CIFAR10. We observe that the momentum MMD flow \eqref{eq:Mom_MMD_GD} converges indeed faster than the MMD flow \eqref{eq:MMD_GD} without momentum.

\begin{figure*}[t!]
\centering
\begin{subfigure}[t]{.45\textwidth}
\includegraphics[width=\linewidth]{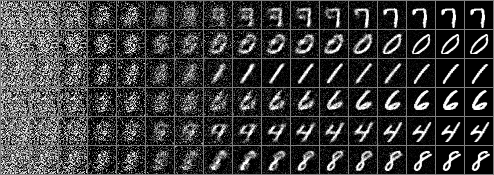}
\end{subfigure}
\hspace{0.1cm}
\begin{subfigure}[t]{.45\textwidth}
\includegraphics[width=\linewidth]{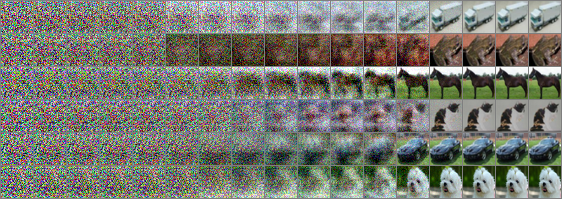}
\end{subfigure}

\begin{subfigure}[t]{.45\textwidth}
\includegraphics[width=\linewidth]{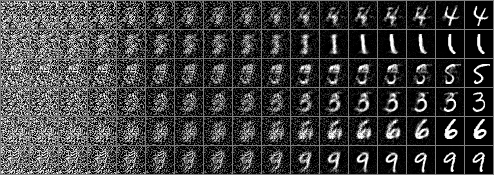}
\end{subfigure}
\hspace{0.1cm}
\begin{subfigure}[t]{.45\textwidth}
\includegraphics[width=\linewidth]{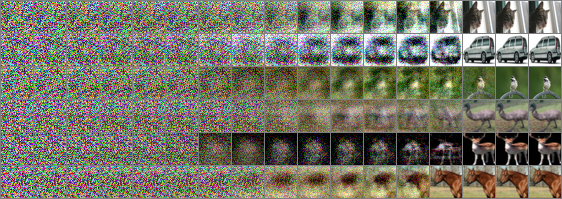}
\end{subfigure}
\caption{Samples and their trajectories from MNIST (left) and CIFAR10 (right) in the MMD flow with momentum (\ref{eq:Mom_MMD_GD}, top) and without momentum (\ref{eq:MMD_GD}, bottom) starting in the uniform distribution on $[0,1]^d$ after $2^k$ steps with $k\in\{0,...,16\}$ (for MNIST) and $k\in\{3,...,19\}$ (for CIFAR10). We observe that the momentum MMD flow \eqref{eq:Mom_MMD_GD} converges faster than the MMD flow \eqref{eq:MMD_GD} without momentum.}
\label{fig:trajectories_momentum}
\end{figure*}

\subsection{Generative MMD Flows}
The (momentum) MMD flows from \eqref{eq:MMD_GD} and \eqref{eq:Mom_MMD_GD} transform samples from the initial distribution $\mu_0$ into samples from the target distribution $\nu$.
Therefore, we propose to train a generative model which approximates these schemes.
The main idea is to approximate the Wasserstein gradient flow $\gamma\colon[0,\infty)\to\mathcal P_2(\mathbb R^d)$ with respect to $\F_\nu$ from \eqref{eq:MMD_ojective_fun} starting at some latent distribution $\mu_0=\gamma(0)$.
Then, we iteratively train neural networks $\Phi_1,...,\Phi_L$ such that $\gamma(t_l)\approx {\Phi_l}_\#\gamma(t_{l-1})$ for some $t_l$ with $0=t_0<t_1<\cdots<t_L$. 
Then, for $t_L$ large enough, it holds $\nu\approx\gamma(t_L)\approx(\Phi_L\circ\cdots\circ\Phi_1)_\#\gamma(0)$ with $\gamma(0)=\mu_0$.
Such methods learning iteratively an "interpolation path" are exploited several times in literature, e.g., \cite{arbel2021annealed,Fan22,ho2020denoising}.
To implement this numerically, we train each network $\Phi_l$ such that it
approximates $T_l$ number of steps from \eqref{eq:MMD_GD} or \eqref{eq:Mom_MMD_GD}. 
The training procedure of our generative MMD flow is summarized in Algorithm~\ref{alg:training_gen_MMD_flows} in Appendix~\ref{app:training_algorithm}.
Once the networks $\Phi_1,...,\Phi_L$ are trained, we can infer a new sample $x$ from our (approximated) target distribution $\nu$ as follows. We draw a sample $x^{(0)}$ from $\mu_0$, compute $x^{(l)}=x^{(l-1)}-\Phi_l(x^{(l-1)})$ for $l=1,...,L$ and set $x=x^{(L)}$. 
In particular, this allows us to simulate paths of the discrepancy flow we have not trained on. 

\begin{remark}[Iterative Training and Sampling]
Since the networks are not trained in an end-to-end fashion but separately, their GPU memory load is relatively low despite a high number of trainable parameters of the full model $(\Phi_l)_{l=1}^L$. This enables training of our model on an 8 GB GPU. Moreover, the training can easily be continued by adding additional networks $\Phi_l$, $l=L+1,...,L'$ to an already trained generative MMD flow $(\Phi_l)_{l=1}^L$, which makes applications more flexible.
\end{remark}

\section{Numerical Examples}\label{sec:5}

\begin{figure*}[b!]
\centering
\begin{subfigure}[t]{.245\textwidth}
\includegraphics[width=\linewidth]{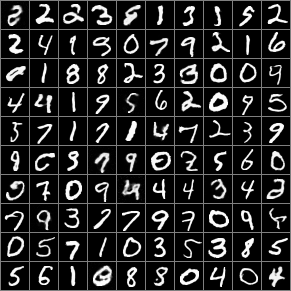}
\end{subfigure}
\hfill
\begin{subfigure}[t]{.245\textwidth}
\includegraphics[width=\linewidth]{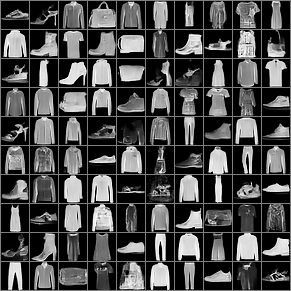}
\end{subfigure}
\hfill
\begin{subfigure}[t]{.245\textwidth}
\includegraphics[width=\linewidth]{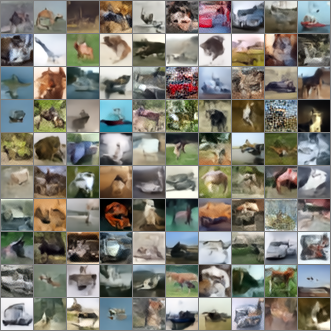}
\end{subfigure}
\hfill
\begin{subfigure}[t]{.245\textwidth}
\includegraphics[width=\linewidth]{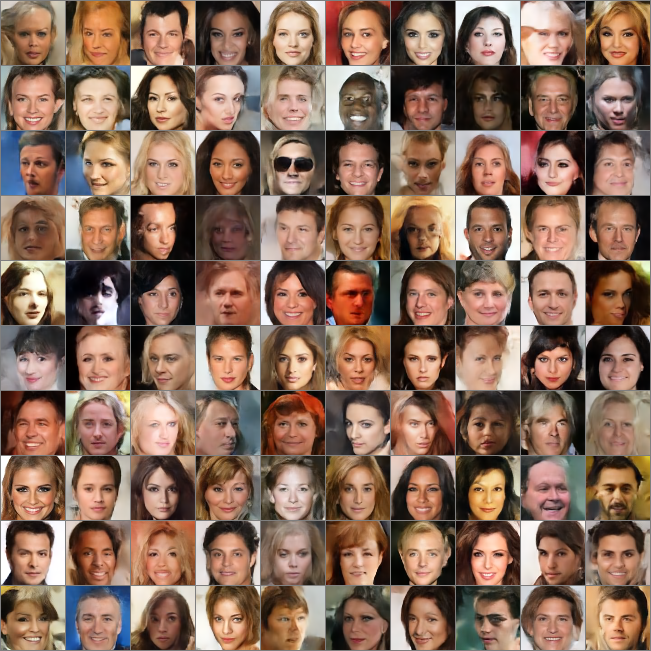}
\end{subfigure}
\caption{
Generated samples of our generative MMD Flow.}
\label{fig:generated_samples}
\end{figure*}

In this section, we apply generative MMD flows for image generation on MNIST, FashionMNIST \citep{XRV2017},CIFAR10 and CelebA \citep{LLWT2015}. The images from MNIST and FashionMNIST are $28\times 28$ gray-value images, while CIFAR10 consists of $32\times32$ RGB images resulting in the dimensions $d=784$ and $d=3072$, respectively. For CelebA, we centercrop the images to $140 \times 140$ and then bicubicely resize them to $64 \times 64$.
We run all experiments either on a single NVIDIA GeForce RTX 3060 or a RTX 4090 GPU with 12GB or 24GB memory,respectively.
To evaluate our results, we use the Fr\'echet inception distance (FID) \citep{HRUNH2017}\footnote{We use the implementation from \url{https://github.com/mseitzer/pytorch-fid}.} between 10K generated samples and the test dataset.
Here, a smaller FID value indicates a higher similarity between generated and test samples.

We choose the networks $(\Phi_l)_{l=1}^L$ to be UNets \citep{RFB2015}, where we use the implementation from \citep{huang2021variational} based on \citep{ho2020denoising}.
Then, we run the generative MMD flow for $L=55$ (MNIST), $L=67$ (FashionMNIST), $L=86$ (CIFAR10) and $L=71$ (CelebA) networks.
The exact setup is described in Appendix~\ref{app:implementation}.
We compare the resulting FIDs with other gradient-flow-based models and various further methods in Table~\ref{table:FID_FASHMNIST}. We computed the standard deviations by independently sampling ten times from one training run and computing the corresponding FID.
We observe that we achieve excellent performance on MNIST and FashionMNIST as well as very good results on CIFAR10 and CelebA. 
Generated samples are given in Figure~\ref{fig:generated_samples} and more samples are given in Appendix~\ref{app:add_examples}. The $L2$-nearest neighbors of the generated samples on MNIST are also illustrated in Figure~\ref{fig:diff_imgs} in Appendix~\ref{app:add_examples}.

\begin{table}[t]
\caption{FID scores for different datasets and various methods.}  
\label{table:FID_FASHMNIST}
\centering
\scalebox{.71}{
\begin{tabular}[t]{c|ccccc} 
                Method                                   & MNIST   & FashionMNIST & CIFAR10  & CelebA \\
\hline
\textit{Auto-encoder based} & & & \\
CWAE \citep{KSTPMJ2020}          & 23.6    & 50.0         & 120.0    & 49.7 \\
SWF+ Autoencoder + RealNVP  \citep{bonet2022efficient}   & 17.8    & 40.6         & -       & 90.9  \\
2-stage VAE \citep{dai2018diagnosing}                    & 12.6    & 29.3         & 72.9    & 44.4  \\
GLF \citep{XYA2019}                                       & 8.2     & 21.3         & 88.3    & 53.2  \\
\hline 
\textit{Adversarial} & & & \\
WGAN \citep{ACB2017,LKMGB2018}                            & 6.7     & 21.5         &  55.2       & 41.3    \\           
MMD GAN \citep{binkowski2018demystifying}                 & 4.2     &  -           & 48.1     & \textbf{29.2}  \\
\hline
\textit{Score-based} & & &\\
NCSN \citep{SE2019}                                       & -       &  -           & \textbf{25.3}  & - \\
\hline
\textit{Flow based} & & & \\
SWF \citep{liutkus19} \footnotemark[3]                   &  225.1  &  207.6       &   -      & 91.2 \\
  
SIG \citep{dai21a}                                        & 4.5     & 13.7         &  66.5      &  37.3 \\
$\ell$-SWF \citep{DLPYL2023}                                          & -       & -            &  59.7 & 38.3 \\
Generative Sliced MMD Flow (ours)                        & \textbf{3.1} $\pm$ 0.06 & \textbf{11.3}$\pm$ 0.07    & 54.8 $\pm$ 0.26  & $ 32.1 \pm$ 0.17  \\
\end{tabular}}
\end{table}

\footnotetext[3]{values taken from \cite{bonet2022efficient}}

\section{Conclusions}\label{sec:6}

\textbf{Discussion.}
We introduced an algorithm to compute (gradients) of the MMD with Riesz kernel efficiently via slicing and sorting reducing the dependence of the computational complexity on the number of particles from $O(NM+N^2)$ to $O((N+M)\log(N+M))$.
For the implementations, we approximated the gradient of sliced MMD by a finite number of slices and proved that the corresponding approximation error depends by a square root on the dimension.
We applied our algorithm for computing MMD flows and approximated them by neural networks. Here, a sequential learning approach ensures computational efficiency.
We included numerical examples for image generation on MNIST, FashionMNIST and CIFAR10. 

\textbf{Limitations.} One of the disadvantages of interacting particle methods is that batching is not easily possible: The particle flow for one set of training points does not give a helpful approximation for another set of training points. This is due to the interaction energy and a general problem of particle flows. Furthermore, taking the projections involves multiplication of every data point with a "full" projection and therefore scales with the dimension $d$. Taking "local" projections like in \citep{DLPYL2023,nguyen2022revisiting} can be much more efficient.

\textbf{Outlook.} Our paper is the first work which utilizes sliced MMD flows for generative modelling. Consequently the approach can be extended in several directions.
Other kernels are considered in the context of slicing in the follow-up paper \citep{H2024}.
From a theoretical viewpoint, the derivative formulas from Theorem~\ref{thm:sorting} can be extended to the non-discrete case by the use of quantile functions, see \citep{BCDP2015,HBGS2023} for some first approaches into this direction.
Towards applications, we could extend the framework to posterior sampling in Bayesian inverse problems.
In this context, the fast computation of MMD gradients can be also of interest for applications which are not based on gradient flows, see e.g., \citet{ardizzone2018analyzing}.
Finally, the consideration of sliced probability metrics is closely related to the Radon transform and is therefore of interest also for non-Euclidean domains like the sphere, see e.g.,~\citep{BBCS2022,QBS2023}.

\subsubsection*{Acknowledgements}
Many thanks to J. Chemseddine for providing parts of the proof of Theorem 2, and to R. Beinert and G. Steidl for fruitful discussions. We thank Mitchell Krock for finding a typo.
J.H.~acknowledges funding by the German Research Foundation (DFG) within the project STE 571/16-1 and by the EPSRC programme grant ``The Mathematics of Deep Learning'' with reference EP/V026259/1, C.W.~by the DFG within the SFB “Tomography Across the Scales” (STE 571/19-1, project number: 495365311),
F.A.~by the DFG under Germany`s Excellence Strategy – The Berlin Mathematics Research Center MATH+ (project AA 5-6),
and P.H. from the DFG within the project SPP 2298 "Theoretical Foundations of Deep Learning".

\bibliographystyle{iclr2024_conference}
\bibliography{references}

\newpage
\appendix

\section{Proof of Theorem~\ref{sliced:unsliced}}\label{proof:sliced:unsliced}
    Let $\mathcal U_{\Sp^{d-1}}$ be the uniform distribution on $\Sp^{d-1}$ and let $\mathrm{k}(x,y) = -|x-y|^r$ for $x,y\in\R, x\neq y$ and $r\in(0,2)$. Moreover, denote by $e=(1,...,0)\in\Sp^{d-1}$ the first unit vector. Then, we have for $x,y \in \R^d$ that
    \begin{align*}
        K(x,y) &:= -\int_{\mathbb S^{d-1}}|\langle \xi,x\rangle - \langle\xi,y\rangle|^r\d \mathcal U_{\Sp^{d-1}}(\xi) = -\|x-y\|^r\int_{\mathbb S^{d-1}}\left|\left\langle \xi,\frac{x-y}{\|x-y\|}\right\rangle\right|^r\d \mathcal U_{\Sp^{d-1}}(\xi)\\
                         &= -\|x-y\|^r\int_{\mathbb S^{d-1}}|\langle\xi,e\rangle|^r\d \mathcal U_{\Sp^{d-1}}(\xi) = -\|x-y\|^r \underbrace{\int_{\mathbb S^{d-1}}|\xi_1|^r\d \mathcal U_{\Sp^{d-1}}(\xi)}_{\eqqcolon c_{d,r}^{-1}}.
    \end{align*}
    It remains to compute the constant $c_{d,r}$ which is straightforward for $d= 1$. For $d>1$ the map
    \[
    (t,\eta)\mapsto (t,\eta\sqrt{1-t^2})
    \]
    defined on $[-1,1)\times\Sp^{d-2}$
    is a parametrization of $\Sp^{d-1}$. The surface measure on $\Sp^{d-1}$ is then given by 
    \[
    \d\sigma_{\Sp^{d-1}}(\xi) = (1-t^2)^{\frac{d-3}{2}}\d\sigma_{\Sp^{d-2}}(\eta)\d t,
    \]
    see \citep[Eq. 1.16]{atkinson2012spherical}. Furthermore, the uniform surface measure $\mathcal U_{\mathbb S^{d-1}}$ reads as 
    \begin{equation}\label{eq:transformation_spherical}
    \d\mathcal U_{\Sp^{d-1}}(\xi)=\frac{1}{s_{d-1}}(1-t^2)^{\frac{d-3}{2}}\d\sigma_{\Sp^{d-2}}(\eta)\d t,
    \end{equation}
    where $s_{d-1}$ is the volume of $\Sp^{d-1}$.
    Hence
    \begin{align}\label{eq:beta}
        c_{d,r}^{-1} &= \int_{\mathbb S^{d-1}}|\xi_1|^r\d \mathcal U_{\Sp^{d-1}}(\xi)
               =\frac{1}{s_{d-1}}\int_{\Sp^{d-2}}\int_{-1}^{1}|t|^r(1-t^2)^{\frac{d-3}{2}}\ \d t d\sigma_{\Sp^{d-2}}(\eta)\\
               &= \frac{s_{d-2}}{s_{d-1}}2\int_0^1t^r(1-t^2)^{\frac{d-3}{2}} \d t
               = \frac{s_{d-2}}{s_{d-1}}B\left(\frac{r+1}{2},\frac{d-1}{2}\right),
    \end{align}
    where $B(z_1,z_2)$ is the beta function and we used the integral identity
    \[
    B(z_1,z_2) = 2\int_0^1 t^{2z_1-1}(1-t^2)^{z_2-1}\d t.
    \]
    Finally, noting that $s_{d-1}=\frac{2\pi^{d/2}}{\Gamma(\frac{d}{2})}$ and $\Beta(z_1,z_2)=\frac{\Gamma(z_1)\Gamma(z_2)}{\Gamma(z_1+z_2)}$, \eqref{eq:beta} can be computed as
    \begin{align*}
    c_{d,r}^{-1}=\frac{\Gamma(\frac{d}{2})}{\sqrt{\pi}\Gamma(\frac{d-1}{2})}\frac{\Gamma(\frac{r+1}{2})\Gamma(\frac{d-1}{2})}{\Gamma(\frac{r+d}{2})}=\frac{\Gamma(\frac{d}{2})\Gamma(\frac{r+1}{2})}{\sqrt{\pi}\Gamma(\frac{d+r}{2})}
    \end{align*}
    Taking the inverse gives the claim. \hfill$\Box$
    
    \begin{remark}[Extension to $\mathcal P_{\frac{r}{2}}(\R^d)$] 
    We can extend Theorem \ref{sliced:unsliced} to $\mathcal P_{\frac{r}{2}}(\R^d)$. To this end, we first show, how we can deduce from \cite[Prop 2.14]{MD2023} that the MMD $\mathcal{D}_{\tilde{K}}$ with respect to the extended Riesz kernel $\tilde K(x,y)=-\|x-y\|^r + \|x\|^r + \|y\|^r$ defines a metric on $\mathcal P_{\frac r 2}(\R^d)$.
    Second, we will see that Theorem \ref{sliced:unsliced} can be extended to $\mathcal P_{\frac{r}{2}}(\R^d)$ as well.

    \begin{enumerate}
    \item[(i)] By \cite[Thm 4.26]{SC2008} the MMD
$$
\mathcal D_{\tilde K}^2(\mu,\nu)=\int\int \tilde K(x,y) d(\mu-\nu)(x)d(\mu-\nu)(y)
$$
is finite on $\mathcal M_{\tilde K}=\{\mu \in \mathcal M(\mathbb R^d):\int \sqrt{\tilde K(x,x)}\mathrm d |\mu|<\infty\}$, where $\mathcal M(\mathbb R^d)$ is the space of all signed measures on $\mathbb R^d$.
Since $\tilde K(x,x)=2\|x\|^r$, we have that $\mathcal M_{\tilde K}=\{\mu \in \mathcal M(\mathbb R^d):\int \|x\|^{\frac r 2}\mathrm d |\mu|<\infty\}$ such that $\mathcal M_{\tilde K}\cap\mathcal P(\mathbb R^d)=\mathcal P_{\frac r 2}(\mathbb R^d)$.
Now, inserting $\tilde K$ \cite[Prop 2.14]{MD2023} states that $\mathcal D_{\tilde{K}}$ is a metric on $\mathcal M_{\tilde K}\cap \mathcal P(\mathbb R^d)=\mathcal P_{\frac r 2}(\mathbb R^d)$, where the assumptions of the proposition are checked in \cite[Ex 2, Ex 3]{MD2023} and $\tilde K$ is named $k_H$ with $H=\frac r 2$.
\item[(ii)] Let $\tilde{\rm K}=c_{d,r}^{-1}\tilde K$ be a rescaled and $\tilde{k}(x,y) = -|x-y|^r+|x|^r+|y|^r$ be the one-dimensional version of the extended Riesz kernel $\tilde K$. With the same calculations as in the proof of Theorem \ref{sliced:unsliced} we can show that
    \begin{align*}
    \tilde{\rm K}(x,y) &= \int_{\mathbb S^{d-1}}\tilde{k}(\langle \xi,x\rangle,\langle\xi,y\rangle)\d \mathcal U_{\Sp^{d-1}}(\xi).
    \end{align*}
    Thus, we also have that $\mathcal S\mathcal D_{\tilde{k}}=\mathcal D_{\tilde{\rm K}}$ on $\mathcal P_{\frac r 2}(\R^d)$.
    \end{enumerate}
    \end{remark}	
		
%---------------------------------------------------------------------------------
\section{Proof of Theorem~\ref{thm:rel}}\label{proof:rel}
%---------------------------------------------------------------------------------
In the first part of the  proof, we will use properties of reproducing kernel Hilbert spaces (RKHS), see \citep{SC2008} for an overview on RKHS.
For the second part, we will need two lemmata. 
The first one gives a definition of the MMD in terms of the Fourier transform
(characteristic function) of the involved measures, where 
$\hat \mu (\xi) = \langle e^{-ix\xi}, \mu \rangle=\int_{\R^d} e^{-i\langle \xi,x\rangle}\d\mu(x)$.
Its proof can be found, e.g., in \citep[Proposition 2]{szekely_energy}. Note that the constant on the right hand-side differs from \citep[Proposition 2]{szekely_energy}, since we use a different notion of the Fourier transform and the constant $\frac12$ in the MMD.

\begin{lemma} \label{fourier}
Let $\mu \in \mathcal P_1(\mathbb{R^d})$ and $K(x,y) = -\|x-y\|$.
Then its holds                                                          
$$
\mathcal D_{K}^2 (\mu,\nu) 
= 
\frac{ \Gamma(\frac{d+1}{2})}{2\pi^\frac{d+1}{2} }
\int_{\mathbb R^d} 
\frac{|\hat \mu(\xi) - \hat \nu(\xi)|^2}{\|\xi\|^{1+d} } \, \d \xi.
$$
\end{lemma}

For the next lemma, recall that the cumulative density functions (cdf)  
of $\mu \in \mathcal P(\R)$ is the function $F_\mu: \R \to [0,1]$ defined by 
$$
F_\mu(x) \coloneqq \mu((-\infty, x]) =
\int_\R\chi_{(-\infty,x]}(y)\d \mu(y),\quad \chi_A(x)=
\begin{cases}
1,&\text{if } x\in A,
\\0,&\text{otherwise}.
\end{cases}
$$
We will need that the \emph{Cramer distance} 
between probability measures $\mu,\nu$ with cdfs 
$F_\mu$ and $F_\nu$
defined by 
$$
\ell_p(\mu,\nu) \coloneqq \left(\int_{\mathbb R} |F_\mu - F_\nu|^p \dx x \right)^\frac{1}{p},
$$
if it exists. The Cramer distance does not exist for arbitrary probability measures.
However, for  $\mu,\nu \in \mathcal P_1(\R)$ is well-known that
\begin{equation}\label{l_1+W_1}
\ell_1(\mu,\nu) = \mathcal W_1(\mu,\nu)
\end{equation}
and we have indeed $F_\mu - F_\nu \in L_2(\R)$. 
The following relation can be found in the literature, see, e.g., \citet{szekely2002}.
However, we prefer to add a proof which clearly shows which assumptions are necessary.

\begin{lemma} \label{l2+D}
Let $\mu, \nu \in \mathcal P_1(\mathbb{R})$ and $\rm{k}(x.y) = -|x-y|$.
Then it holds
$$
\ell_2(\mu,\nu) = \mathcal D_{\rm{k}} (\mu,\nu).$$
\end{lemma}

\begin{proof}
By Lemma \ref{fourier},  we know that 
\begin{equation}\label{Fourier}
\mathcal D_{\text{k}}^2 (\mu,\nu) 
= - \frac12 \int_{\R}\int_{\R}   |x-y| \, \dx (\mu - \nu) (x) \dx (\mu - \nu) (y)
= \frac{1}{2\pi} \int_{\R}\frac{ |\hat \mu (\xi) - \hat \nu(\xi)|^2}{\xi^2} \, \dx \xi.
\end{equation}
For $\mu, \nu \in \mathcal P_1(\mathbb{R})$, we have
$F_\mu - F_\nu \in L_2(\mathbb R)$ and can apply
Parseval's equality
\begin{equation}\label{Parseval}
\ell_2^2(\mu,\nu) = \int_{\R} |F_\mu(t) - F_\nu(t)|^2 \, \dx t = \frac{1}{2 \pi}\int_{\R} |\hat F_\mu(\xi) - \hat F_\nu(\xi)|^2 \, \dx \xi.
\end{equation}
Now we have for the distributional derivative of $F_\mu$ that
$
D F_\mu = \mu, 
$
which can be seen as follows:
using Fubini's theorem, 
we have for any Schwartz function $\phi \in \mathcal S(\R)$ that
\begin{align*}
  \langle D F_\mu, \phi \rangle
  = - \langle F_\mu, \phi' \rangle
  &= - \int_\R \int_{\R} \chi_{(-\infty,x]}(y) \,  \phi'(x) \, \dx \mu(y) \, \dx x
  = - \int_\R \int_{\R} \chi_{(-\infty,x]}(y) \,  \phi'(x) \,  \dx x \, \dx \mu(y)
  \\
  &= - \int_\R \int_y^\infty \phi'(x) \, \dx x \, \dx \mu(y)
    =\int_\R \phi(y) \, \dx \mu(y)
    = \langle \mu, \phi \rangle.
\end{align*}
Then we obtain by the differentiation property of the Fourier transform \citep{PPST2018} that
$$
\hat \mu (\xi) =
\widehat{D F_\mu}(\xi) = - i\xi \hat F_\mu(\xi).
$$
Finally, \eqref{Parseval} becomes
$$
\ell_2^2(\mu,\nu) = \frac{1}{2 \pi}\int_{\R} \frac{|\hat \mu(\xi) - \hat \nu(\xi)|^2}{\xi^2} \, \dx \xi,
$$
which yields  the assertion by \eqref{Fourier}.
\end{proof}

Now we can prove Theorem~\ref{thm:rel}.

\begin{proof}                                                                        
1. To prove of the first inequality, we use the reproducing property 
$$
\langle f,\tilde K(x,\cdot) \rangle_{\mathcal{H}_{\tilde K}} = f(x)
$$ 
of the kernel in the associated RKHS $\mathcal{H}_{\tilde K}$.
For any $f\in \mathcal{H}_{\tilde K}$ 
with $\|f\|_{\mathcal{H}_{\tilde K}} \leq 1$ and any $\pi \in \Pi(\mu,\nu)$, we  use the estimation
\begin{align*}
\Big| \int_{\R^d} f(x) \d (\mu -\nu)(x) \Big| 
&= 
\Big| \int_{\R^d} \int_{\R^d} f(x) -f(y) \d \pi(x,y) \Big|  
\leq 
\int_{\R^d} \int_{\R^d} |f(x) -f(y)| \d \pi(x,y) \\
&= 
\int_{\R^d} \int_{\R^d} |\langle f, \tilde K(x,\cdot) - \tilde K(y,\cdot) \rangle_{\mathcal{H}_{\tilde K}} | \d \pi(x,y)\\
&\leq 
\int_{\R^d} \int_{\R^d} \|\tilde K(x,\cdot) -\tilde K(y,\cdot)\|_{\mathcal{H}_{\tilde K}} \d \pi(x,y), 
\end{align*}
which is called ``coupling bound'' in \citep[Prop.~20]{SGFSL2010}.
Then, since 
$\|\tilde K(x,\cdot) -\tilde K(y,\cdot)\|_{\mathcal{H}_{\tilde K}}^2 = \tilde K(x,x) + \tilde K(y,y) - 2\tilde K(x,y)
= 2 \|x-y\|$
and using Jensen's inequality for the concave function $\sqrt{\cdot}$, we obtain
\begin{align*}
\Big| \int_{\R^d} f(x) \d (\mu -\nu)(x) \Big| 
&\leq
\sqrt{2} \int_{\R^d} \int_{\R^d} \|x -y\|^{\frac{1}{2}} \d \pi(x,y)  
\leq 
\left(2 \int_{\R^d} \int_{\R^d} \|x -y\|\d \pi(x,y)\right)^{\frac{1}{2}}.
\end{align*}
By the dual definition of the discrepancy 
$2 \mathcal D_K (\mu,\nu) = \sup_{\|f\|_{\mathcal{H}_{\tilde K}} \le 1} \int_{\R^d} f \d (\mu - \nu)$, see \cite{NW2010},
and 
taking the supremum over all such $f$ and the infimum over all $\pi \in \Pi(\mu,\nu)$, we finally arrive at
$$
2 \mathcal D_K^2 (\mu,\nu) \le \mathcal W_1(\mu,\nu).
$$
2. The second inequality can be seen as follows: 
by \citet[Lemma 5.1.4]{bonotte2013}, there exists a constant $c_d >0$ such that
\begin{align*}
     W_{1}(\mu, \nu)  
		&\leq 
		c_d R^{\frac{d}{d+1}} \mathcal{S} \mathcal W_1(\mu, \nu)^{\frac{1}{d+1}}
    = 
		c_d R^{\frac{d}{d+1}} 
		\left( 
		\mathbb{E}_{\xi \sim \mathcal{U}_{\mathbb{S}^{d-1}}}
		\left[\mathcal W_1 \left(P_{\xi_{\#}} \mu, P_{\xi_{\#}} \nu\right) \right] 
		\right)^{\frac{1}{d+1}}.
\end{align*}
Further, we obtain by	\eqref{l_1+W_1}, the Cauchy-Schwarz inequality
and Lemma \ref{l2+D} that
\begin{align*}
  \mathbb{E}_{\xi \sim \mathcal{U}_{\mathbb{S}^{d-1}}}
		\left[W_1 \left(P_{\xi_{\#}} \mu, P_{\xi_{\#}} \nu\right) \right]
		&=		
		\mathbb{E}_{ \xi \sim \mathcal{U}_{ \mathbb{S}^{d-1} } }
		\left[ l_1 \left( P_{\xi_\#} \mu, P_{\xi_{\#}} \nu \right) \right]
		 \\
		&\leq 
				\mathbb{E}_{ \xi \sim \mathcal{U}_{\mathbb{S}^{d-1} } } 
		\left[(2R)^{\frac{1}{2}}
		l_2 \left(P_{\xi_{\#}} \mu, P_{\xi_{\#}} \nu \right) 
		\right]
		     \\
			&=  
			(2R)^{\frac{1}{2}}
			\mathbb{E}_{\xi \sim \mathcal{U}_{\mathbb{S}^{d-1}}}
		\left[\mathcal{D}_{\text{k}}\left(P_{\xi_{\#}} \mu, P_{\xi_{\#}} \nu\right)\right]
	     \\
		&\leq 
		(2R)^{\frac{1}{2}} 
				\left( \mathbb{E}_{\xi \sim \mathcal{U}_{\mathbb{S}^{d-1}}}
				\left[\mathcal{D}^2_{{\text{k}}}
		\left(P_{\xi_{\#}} \mu, P_{\xi_{\#}} \nu\right)
		\right] \right)^\frac12,
\end{align*} 
and finally by Theorem \ref{sliced:unsliced} that
\begin{align*}
  \mathbb{E}_{\xi \sim \mathcal{U}_{\mathbb{S}^{d-1}}}
		\left[W_1 \left(P_{\xi_{\#}} \mu, P_{\xi_{\#}} \nu\right) \right]
  &\leq (2R)^{\frac{1}{2}}  \mathcal{D}_{\mathrm{K}}(\mu, \nu)
   = (2R)^{\frac{1}{2}}
  \left( \frac{\Gamma(\frac{d}{2})}{\sqrt{\pi}\Gamma(\frac{d+1}{2})} \right)^{\frac{1}{2}}  \mathcal{D}_{K}(\mu, \nu).
\end{align*} 
In summary, this results in
$$ W_{1}(\mu, \nu)  
		\leq 
		c_d \left(\frac{2 \Gamma(\frac{d}{2})}{\sqrt{\pi}\Gamma(\frac{d+1}{2})}  \right)^{\frac{1}{2(d+1)}}  R^{\frac{2d+1}{2d+2}}
		\mathcal{D}_{K}(\mu, \nu)^{\frac{1}{d+1}}.
$$
\end{proof}

Under some additional assumptions, similar bounds have been considered in \cite{Chaifai2018} for the Coloumb kernel $K(x,y)=\|x-y\|^{2-d}$.

%---------------------------------------------------------------------------------
\section{Proof of Theorem~\ref{thm:sorting}}\label{proof:sorting}

\textbf{Interaction Energy:}
This part of the proof is similar to \citep[Sec.~3]{TSGSW2011}.
Using that $\sigma$ is a permutation and by reordering the terms in the double sum, we can rewrite the interaction energy by
\begin{align}
E(\zb x)&=-\frac{1}{2N^2}\sum_{i=1}^N\sum_{j=1}^N |x_i-x_j|=-\frac{1}{2N^2}\sum_{i=1}^N\sum_{j=1}^N|x_{\sigma(i)}-x_{\sigma(j)}|\\
&=-\frac{1}{N^2}\sum_{i=1}^N\sum_{j=i+1}^Nx_{\sigma(j)}-x_{\sigma(i)}=\sum_{i=1}^N \frac{N-(2i-1)}{N^2}x_{\sigma(i)}=\sum_{i=1}^N \frac{N-(2\sigma^{-1}(i)-1)}{N^2}x_{i}.
\end{align}
Since the $x_i$are pairwise disjoint,  the sorting permutation $\sigma$ is constant in a neighborhood of $\zb x$.
Hence, $E$ is differentiable with derivative 
\begin{align*}
\nabla_{x_i}E(\zb x)=\frac{N+1-2\sigma^{-1}(i)}{N^2}.
\end{align*}

\textbf{Potential Energy:}
For any $x\neq y\in\R$ it holds
\begin{align*}
\nabla_{x}|x-y|=\chi(x,y),\quad\text{where}\quad\chi(x,y)=\begin{cases}1,&$if $x>y,\\-1,&$if $x<y.\end{cases}
\end{align*}
Thus, we have that
\begin{align}
\nabla_{x_i}V(\zb x|\zb y)&=\frac1{MN}\sum_{j=1}^M\chi(x_i,y_j)\\&=\frac{1}{MN}\big(\#\{j\in \{1,...,M\}:y_j<x_i\}-\#\{j\in \{1,...,M\}:y_j>x_i\}\big)
\end{align}
Using that $\#\{j\in \{1,...,M\}:y_j>x_i\}=M-\#\{j\in \{1,...,M\}:y_j<x_i\}$ the above expression is equal to
\begin{align}
\frac{1}{MN}\big(2\,\#\{j\in \{1,...,M\}:y_j<x_i\} - M\Big)=\frac{2\,\#\{j\in \{1,...,M\}:y_j<x_i\} - M}{MN}.
\end{align}
\phantom{.}\hfill$\Box$

\section{Proof of Theorem~\ref{thm:convergence_rate}}\label{proof:concentration}
In this section, we derive error bounds for the stochastic estimators for the gradient of MMD as defined in \eqref{eq:def_grad_est1} for the Riesz kernel with $r=1$.
To this end, we employ concentration inequalities \citep{vershynin2018high}, which were generalized to vector-valued random variables in \citet{kohler2017sub}. 
We will need the following Lemma which is the Bernstein inequality and is stated in \citep[Lemma 18]{kohler2017sub}.

\begin{lemma}[Bernstein Inequality]\label{D:bernstein}
Let $X_1,...,X_P$ be independent random vectors with mean zero,$\Vert X_i \Vert \leq \mu$ and $\mathbb{E}[\Vert X_i\Vert_2^2] \leq \sigma^2$. Then it holds for $0 < t < \frac{\sigma^2}{\mu}$ that 
\begin{align*}
\mathbb{P}\Big[\Big\Vert \frac{1}{P} \sum_{i=1}^P X_i\Big\Vert_2 > t\Big] \leq \exp{\left(-\frac{P\ t^2}{8 \sigma^2}+\frac{1}{4}\right)}.
\end{align*}
\end{lemma}

Now we can use Lemma~\ref{D:bernstein} to show convergence of the finite sum approximation to the exact gradient.

\begin{theorem}[Concentration Inequality]\label{thm:concentration_bound}
Let $x_1,...x_N,y_1,...,y_M \in \R^d$. Then, it holds 
\begin{align}
\mathbb{P}\Big[\|\tilde \nabla_P F_d (\zb x | \zb y)-\nabla F_d (\zb x | \zb y)\|>t\Big]
\leq\exp{\Big(-\frac{P\ t^2}{32\  (c_d+1)^2}+\frac{1}{4}\Big)}.
\end{align}
\end{theorem}

\begin{proof}
Let $\xi_1,...,\xi_P\sim\mathcal U_{\Sp^{d-1}}$ be the independent random variables from the definition of $\tilde \nabla_P$ in \eqref{eq:def_grad_est1}.
We set
\begin{align*}
X_{i,l}\coloneqq c_d\nabla_l F_1(\langle \xi_i,x_1\rangle,...,\langle \xi_i,x_N\rangle|\langle \xi_i,y_1\rangle,...,\langle \xi_i,y_M\rangle) \xi_i
\end{align*}
and define the $dN$-dimensional random vector $X_i=(X_{i,1},\cdots X_{i,N})$.
Then, we have by \eqref{eq:derivative_MMD} that 
\begin{align*}
\mathbb{E}[X_i] &=   \nabla F_d (\zb x | \zb y) 
=\big(\nabla_{x_1} F_d (\zb x | \zb y),\cdots,\nabla_{x_N} F_d (\zb x | \zb y) \big).
\end{align*}
By Theorem~\ref{thm:sorting}, we know that $\|X_{i,l}\|_2\leq \frac{2c_d}{N}$, and by \eqref{eq:discrete_MMD} it holds
\begin{align*}
\|\mathbb E \big[ X_{i,l}\big]\|_2=\|\nabla_{x_l}F_d (\zb x | \zb y) \|_2 \leq \frac{2}{N}.
\end{align*}
Let $ \tilde{X}_i = X_i - \mathbb{E}[X_i]$. Then it holds
\[
    \| \tilde{X}_i\|_2 
    \leq \sum_{l=1}^N \|\tilde{X}_{i,l}\|_2 \leq \sum_{l=1}^N \frac{2c_d+2}{N} = 2c_d + 2
\]
and thus $\mathbb{E} \big[\|\tilde{X}_i\|_{2}^2\big] \leq 4(c_d+1)^2$.
Since we have $\mathbb{E}[\tilde{X}_i] = 0$ for all $i=1,...,P$, we can apply Lemma~\ref{D:bernstein} and obtain 
\begin{align*}
\mathbb{P}\Big[\Big\Vert \frac{1}{P} \sum_{i=1}^P X_i - \nabla F_d (\zb x |\zb y) \Big\Vert >t\Big] 
&= \mathbb{P}\Big[\Big\Vert \frac{1}{P} \sum_{i=1}^P \tilde{X}_i \Big\Vert >t\Big]
\leq \exp{\left(-\frac{P\ t^2}{32\ (c_d+1)^2}+\frac{1}{4}\right)}.
\end{align*}
Since we have by definition that
\begin{align*}
\frac{1}{P}\sum_{i=1}^P X_i=\tilde \nabla_P F_d (\zb x | \zb y)
\end{align*}
this yields the assertion.
\end{proof}
Finally we can draw a corollary which immediately shows Theorem \ref{thm:convergence_rate}.

\begin{corollary}[Error Bound for Stochastic Gradients]\label{cor:concentration}
For $x_1,...,x_N,y_1,...,y_M\in\R^d$, it holds
\begin{align*}
\E[\|\tilde\nabla_P F_d(\zb x|\zb y)-\nabla F_d(\zb x|\zb y)\|]\leq \frac{\exp(1/4)\sqrt{32}\pi(\sqrt{d}+1)}{2\sqrt{2\ P}}.
\end{align*}
\end{corollary}

\begin{proof}
Denote by $X$ the random variable
\begin{align*}
X=\|\tilde\nabla_P F_d(\zb x|\zb x)-\nabla F_d(\zb x|\zb y)\|.
\end{align*}
Then, we have by Theorem~\ref{thm:concentration_bound} that
\begin{align*}
\mathbb P[X>t]\leq \exp{\left(-\frac{P\ t^2}{32\ (c_d+1)^2}+\frac{1}{4}\right)}.
\end{align*}
Thus, we obtain
\begin{align*}
\E[X]=\int_0^\infty \mathbb P[X>t]\d t\leq \exp(1/4)\int_0^\infty \exp{\left(-\frac{P\ t^2}{32\ (c_d+1)^2}\right)}\d t=\frac{\exp(1/4)\sqrt{32\pi}(c_d+1)}{2\sqrt{P}},
\end{align*}
where the last step follows from the identity $\int_0^\infty \exp(-t^2)\d t=\frac{\sqrt{\pi}}{2}$.

Now we proceed to bound the constant $c_d$ in the dimensions.
By  Theorem~\ref{sliced:unsliced} we have that 
$c_{d} =  \frac{\sqrt{\pi}\Gamma(\frac{d+1}{2})}{\Gamma(\frac{d}{2})}$.
Now the claim follows from the bound $\frac{\Gamma(\frac{d+1}{2})}{\Gamma(\frac{d}{2})} \leq \sqrt{\frac{d}{2}+\sqrt{\frac{3}{4}}-1}$ proven by \cite{Kershaw83}.
\end{proof}

\section{Comparison of Different Kernels in MMD} \label{app:motivation_riesz}

Next we compare the MMD flows of the Riesz kernels
with those of the positive definite kernels
\begin{align}
K_{\text{G}} (x,y) &\coloneqq \exp \Big( -\frac{\Vert x - y \Vert^2}{2 \sigma^2} \Big) 
\quad \text{(Gaussian)},
\\
K_{\text{IM}} (x,y) &\coloneqq \frac{1}{\sqrt{\Vert x - y \Vert^2 + c}}
\quad \text{(Inverse Multiquadric)},\\
K_{\text{L}} (x,y) &\coloneqq \exp \Big( -\frac{\Vert x - y \Vert}{\sigma} \Big) \quad \text{(Laplacian)}.
\end{align}
The target distribution is defined as the uniform distribution on three non-overlapping circles and the initial particles are drawn from a Gaussian distribution with standard deviation $0.01$, compare \citet{GAG2021}.
We recognize  in Figures \ref{fig:different_kernels1} 
and \ref{fig:different_kernels2}
that 
in contrast to the Riesz kernel, the other MMD flows
\begin{itemize}
 \item 
 heavily depend on the parameters
$\sigma$ and $c$ (stability against parameter choice);
\item cannot perfectly recover the uniform distribution;
zoom into the middle right circles to see that small blue parts are not covered (approximation of target distribution).
\end{itemize}

Moreover, in contrast to the other kernels, the Riesz kernel with $r=1$ is positively homogeneous such that the MMD flow is equivariant against scalings of initial and target measure.

Finally, it is interesting that the Riesz kernel 
is related to the Brownian motion by the following remark, see also \citet{MD2023} for the more general fractional Brownian motion.

\begin{remark}\label{rem:riesz_BM}
In the one-dimensional case, the extended Riesz kernel with $r=1$ reads as
$$
K(x,y)=-|x-y|+|x|+|y|=2\min(x,y),
$$
which is the twice the covariance kernel of the Brownian motion. More precisely, let $(W_t)_{t>0}$ be a Brownian motion and $s,t>0$. Then, it holds
$$
\mathrm{Cov}(W_s,W_t)=\min(s,t)=\frac12K(x,y).
$$
\end{remark}

\begin{figure*}[t!]
\centering
Gaussian kernel
\begin{minipage}{\linewidth}
\rotatebox{90}{
\begin{minipage}{0.12\linewidth}
\centering
$\sigma^2=0.05$
\end{minipage}
}
\begin{subfigure}[t]{.23\textwidth}
\includegraphics[width=\linewidth]{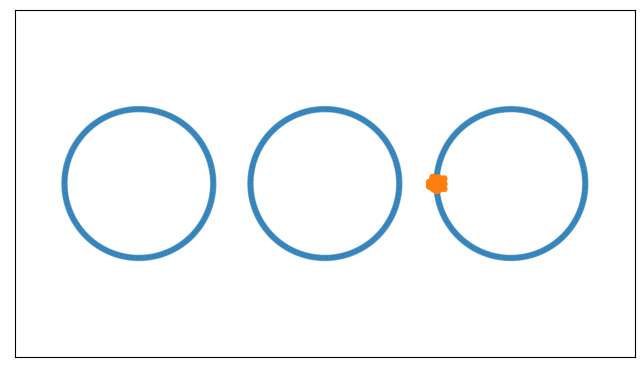}
\end{subfigure}
\hfill
\begin{subfigure}[t]{.23\textwidth}
\includegraphics[width=\linewidth]{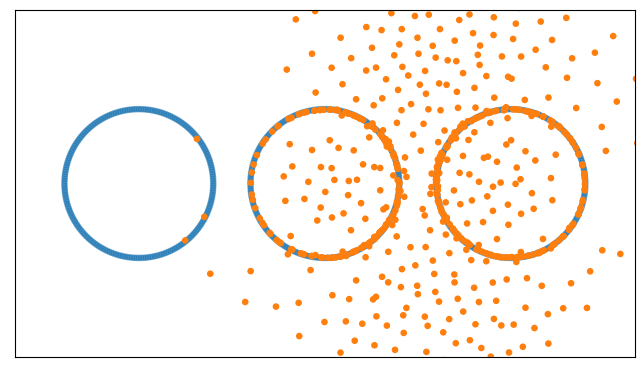}
\end{subfigure}
\hfill
\begin{subfigure}[t]{.23\textwidth}
\includegraphics[width=\linewidth]{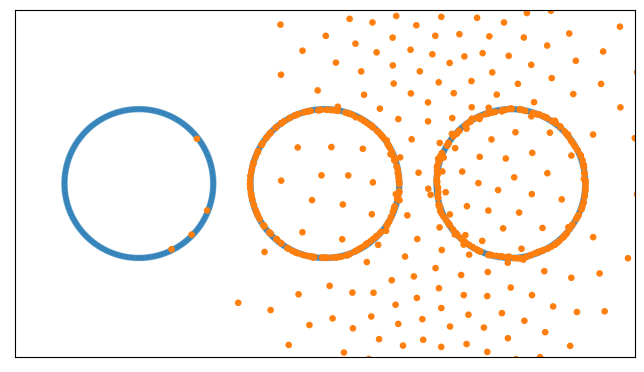}
\end{subfigure}
\hfill
\begin{subfigure}[t]{.23\textwidth}
\includegraphics[width=\linewidth]{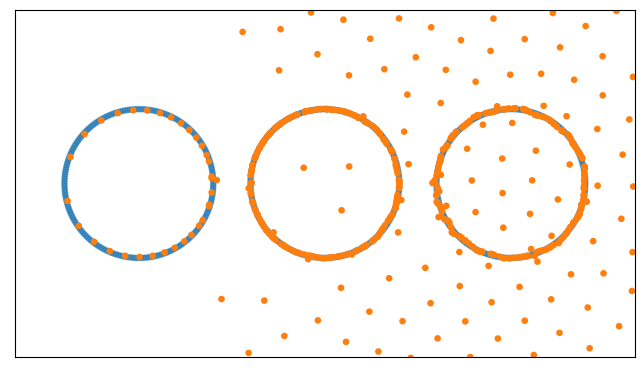}
\end{subfigure}
\end{minipage}

\begin{minipage}{\linewidth}
\rotatebox{90}{
\begin{minipage}{0.125\linewidth}
\centering
$\sigma^2=0.1$
\end{minipage}
}
\begin{subfigure}[t]{.23\textwidth}
\includegraphics[width=\linewidth]{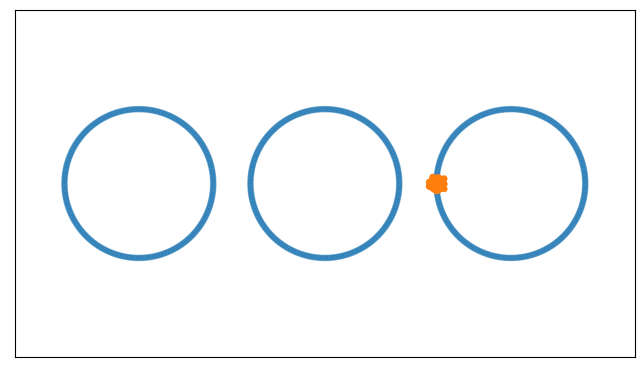}
\end{subfigure}
\hfill
\begin{subfigure}[t]{.23\textwidth}
\includegraphics[width=\linewidth]{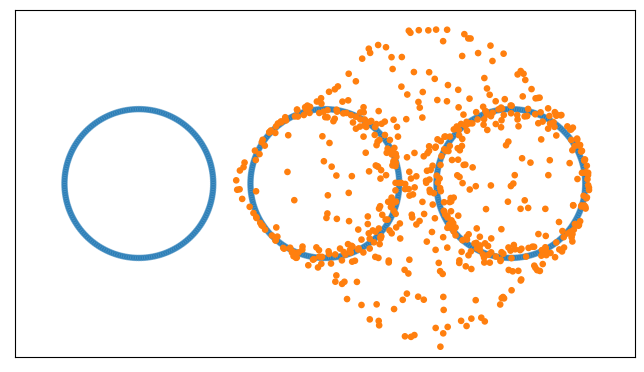}
\end{subfigure}
\hfill
\begin{subfigure}[t]{.23\textwidth}
\includegraphics[width=\linewidth]{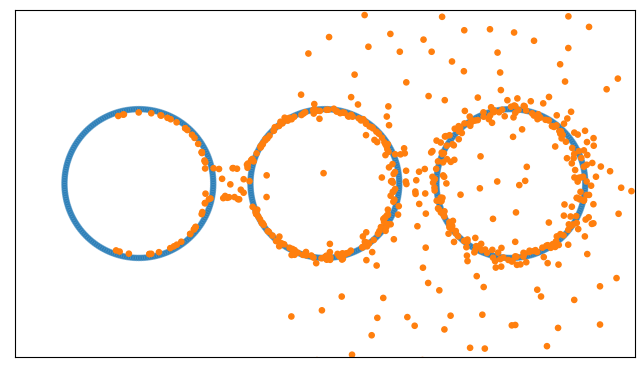}
\end{subfigure}
\hfill
\begin{subfigure}[t]{.23\textwidth}
\includegraphics[width=\linewidth]{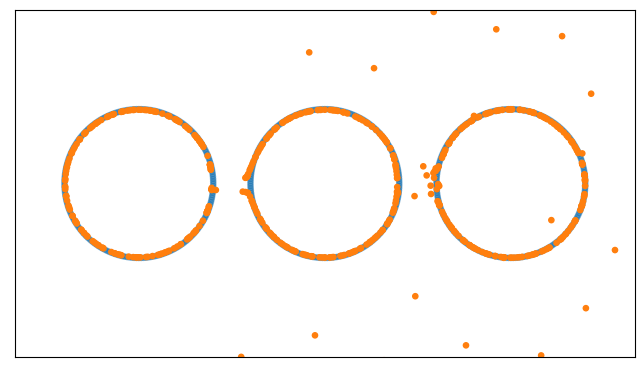}
\end{subfigure}
\end{minipage}

\begin{minipage}{\linewidth}
\rotatebox{90}{
\begin{minipage}{0.125\linewidth}
\centering
$\sigma^2=0.5$
\end{minipage}
}
\begin{subfigure}[t]{.23\textwidth}
\includegraphics[width=\linewidth]{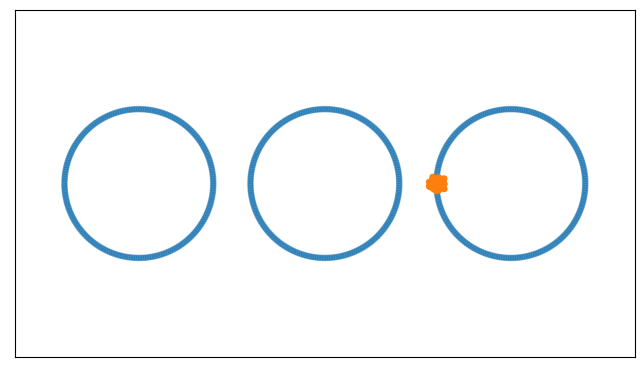}
\caption*{T=0}
\end{subfigure}
\hfill
\begin{subfigure}[t]{.23\textwidth}
\includegraphics[width=\linewidth]{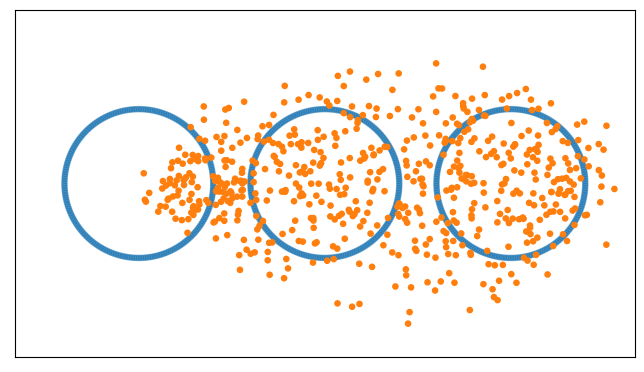}
\caption*{T=2}
\end{subfigure}
\hfill
\begin{subfigure}[t]{.23\textwidth}
\includegraphics[width=\linewidth]{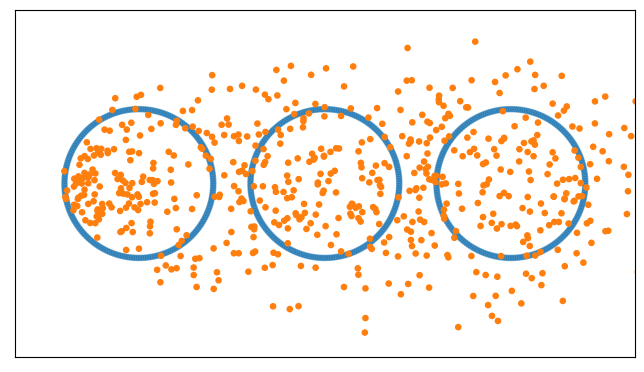}
\caption*{T=10}
\end{subfigure}
\hfill
\begin{subfigure}[t]{.23\textwidth}
\includegraphics[width=\linewidth]{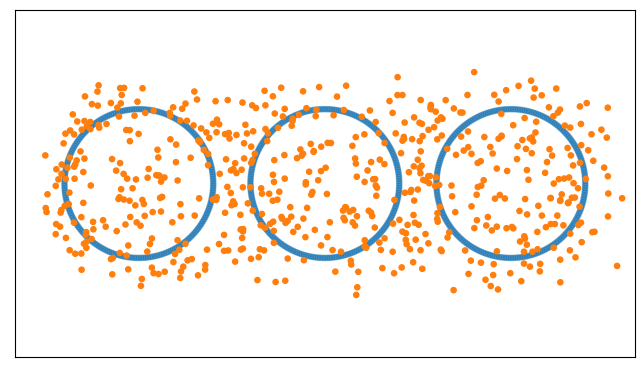}
\caption*{T=100}
\end{subfigure}
\end{minipage}

\vspace{0.3cm}
\centering
Inverse multiquadric kernel

\begin{minipage}{\linewidth}
\rotatebox{90}{
\begin{minipage}{0.125\linewidth}
\centering
$c = 0.01$
\end{minipage}
}
\begin{subfigure}[t]{.23\textwidth}
\includegraphics[width=\linewidth]{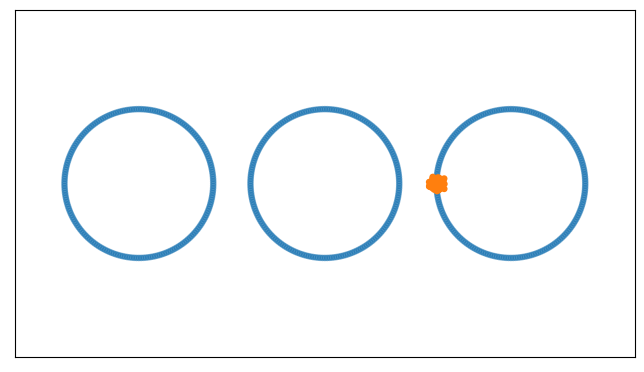}
\end{subfigure}
\hfill
\begin{subfigure}[t]{.23\textwidth}
\includegraphics[width=\linewidth]{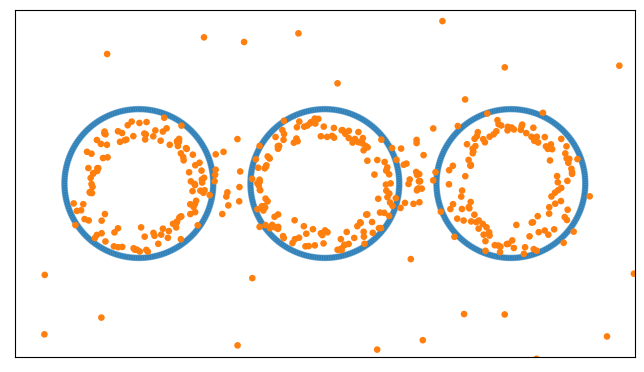}
\end{subfigure}
\hfill
\begin{subfigure}[t]{.23\textwidth}
\includegraphics[width=\linewidth]{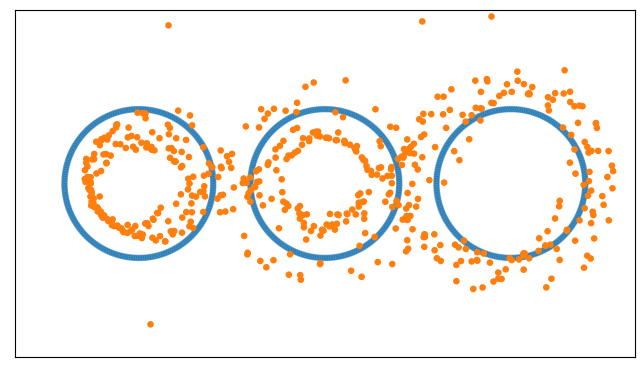}
\end{subfigure}
\hfill
\begin{subfigure}[t]{.23\textwidth}
\includegraphics[width=\linewidth]{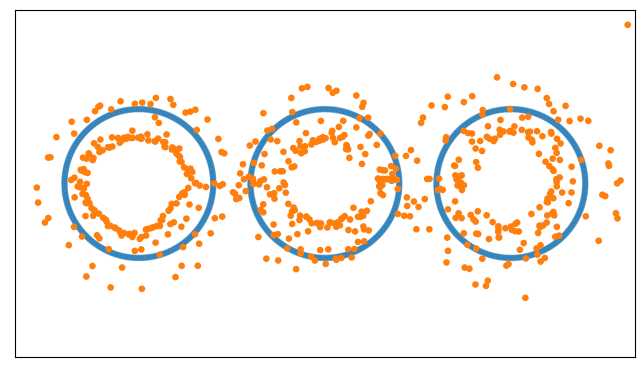}
\end{subfigure}
\end{minipage}

\begin{minipage}{\linewidth}
\rotatebox{90}{
\begin{minipage}{0.125\linewidth}
\centering
$c = 0.05$
\end{minipage}
}
\begin{subfigure}[t]{.23\textwidth}
\includegraphics[width=\linewidth]{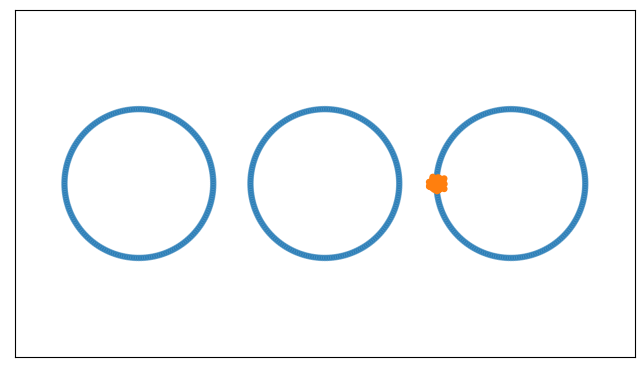}
\end{subfigure}
\hfill
\begin{subfigure}[t]{.23\textwidth}
\includegraphics[width=\linewidth]{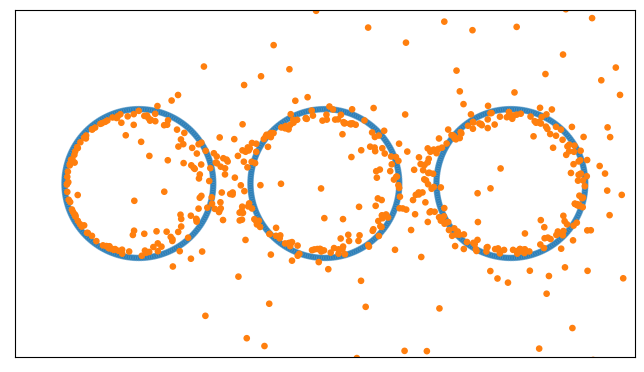}
\end{subfigure}
\hfill
\begin{subfigure}[t]{.23\textwidth}
\includegraphics[width=\linewidth]{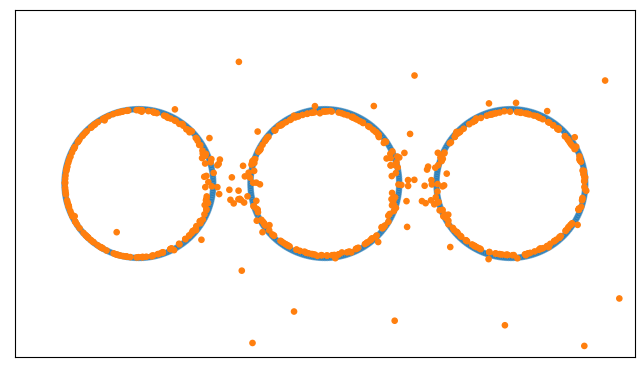}
\end{subfigure}
\hfill
\begin{subfigure}[t]{.23\textwidth}
\includegraphics[width=\linewidth]{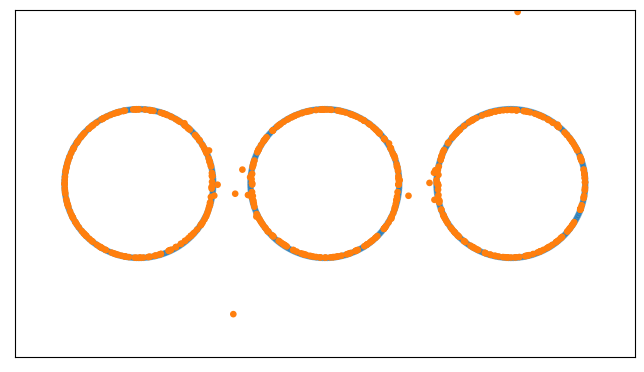}
\end{subfigure}
\end{minipage}

\begin{minipage}{\linewidth}
\rotatebox{90}{
\begin{minipage}{0.125\linewidth}
\centering
$c = 0.5$
\end{minipage}
}
\begin{subfigure}[t]{.23\textwidth}
\includegraphics[width=\linewidth]{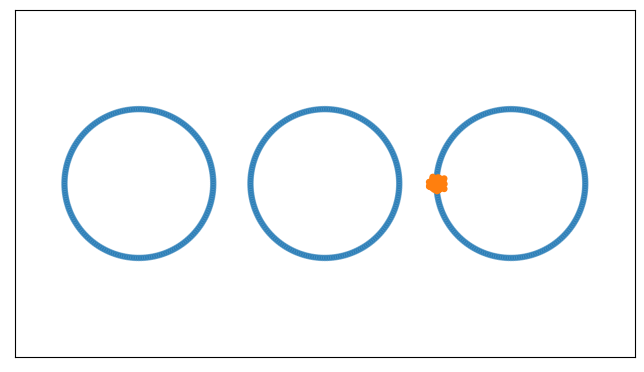}
\caption*{T=0}
\end{subfigure}
\hfill
\begin{subfigure}[t]{.23\textwidth}
\includegraphics[width=\linewidth]{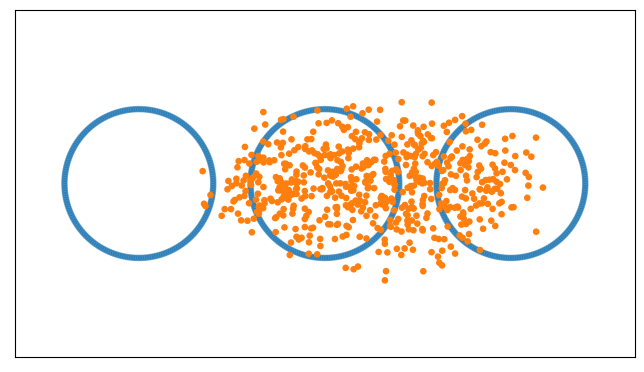}
\caption*{T=2}
\end{subfigure}
\hfill
\begin{subfigure}[t]{.23\textwidth}
\includegraphics[width=\linewidth]{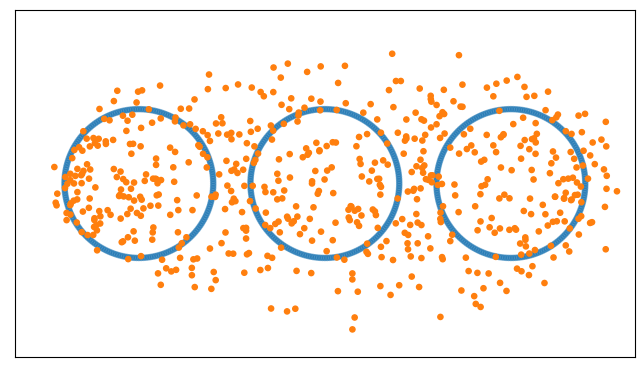}
\caption*{T=10}
\end{subfigure}
\hfill
\begin{subfigure}[t]{.23\textwidth}
\includegraphics[width=\linewidth]{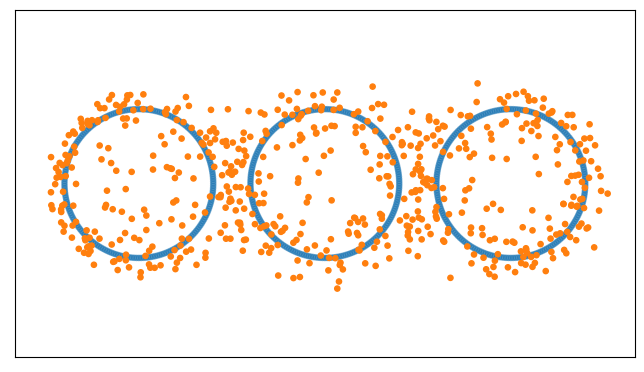}
\caption*{T=100}
\end{subfigure}
\end{minipage}
\vspace{0.3cm}
\centering
\caption{Comparison of the MMD flow with Gaussian kernel (top) and inverse multiquadric kernel (bottom) for different hyperparameters.} \label{fig:different_kernels1}
\end{figure*}

\begin{figure*}[t!]
\centering
Laplace kernel
\begin{minipage}{\linewidth}
\rotatebox{90}{
\begin{minipage}{0.125\linewidth}
\centering
$\sigma^2 = 0.01$
\end{minipage}
}
\begin{subfigure}[t]{.23\textwidth}
\includegraphics[width=\linewidth]{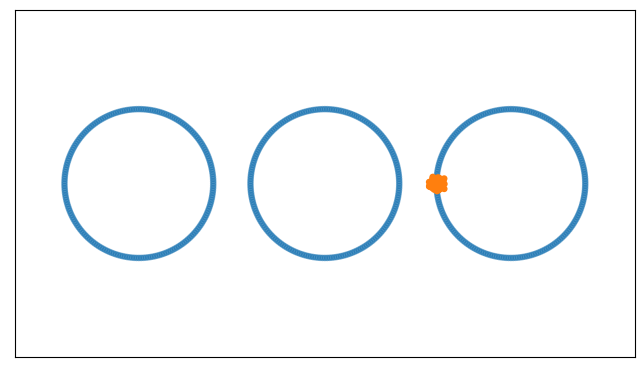}
\end{subfigure}
\hfill
\begin{subfigure}[t]{.23\textwidth}
\includegraphics[width=\linewidth]{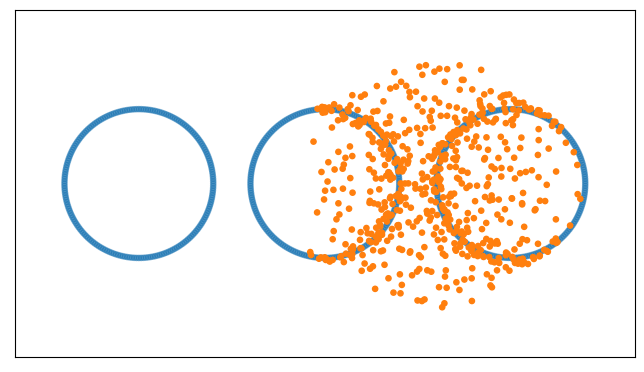}
\end{subfigure}
\hfill
\begin{subfigure}[t]{.23\textwidth}
\includegraphics[width=\linewidth]{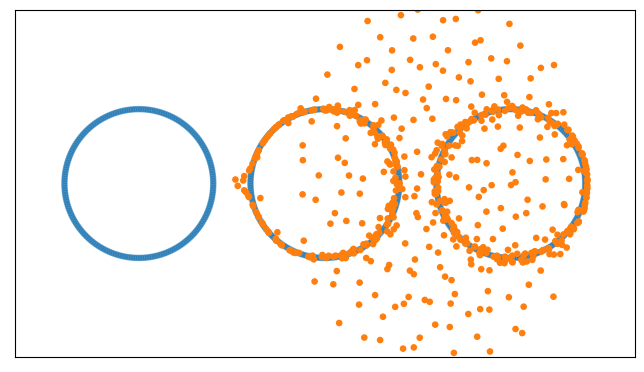}
\end{subfigure}
\hfill
\begin{subfigure}[t]{.23\textwidth}
\includegraphics[width=\linewidth]{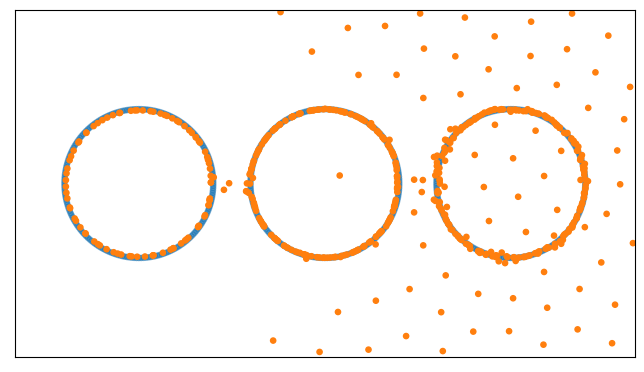}
\end{subfigure}
\end{minipage}

\begin{minipage}{\linewidth}
\rotatebox{90}{
\begin{minipage}{0.125\linewidth}
\centering
$\sigma^2 = 0.02$
\end{minipage}
}
\begin{subfigure}[t]{.23\textwidth}
\includegraphics[width=\linewidth]{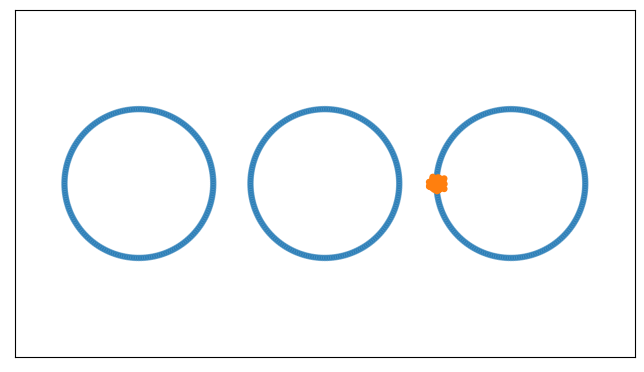}
\end{subfigure}
\hfill
\begin{subfigure}[t]{.23\textwidth}
\includegraphics[width=\linewidth]{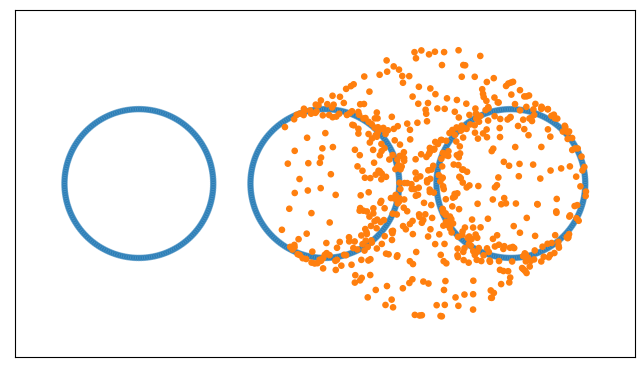}
\end{subfigure}
\hfill
\begin{subfigure}[t]{.23\textwidth}
\includegraphics[width=\linewidth]{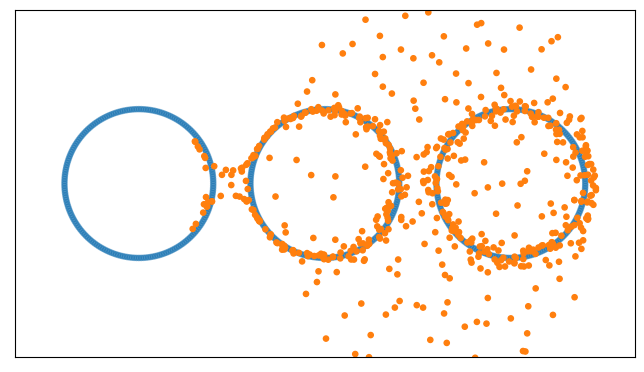}
\end{subfigure}
\hfill
\begin{subfigure}[t]{.23\textwidth}
\includegraphics[width=\linewidth]{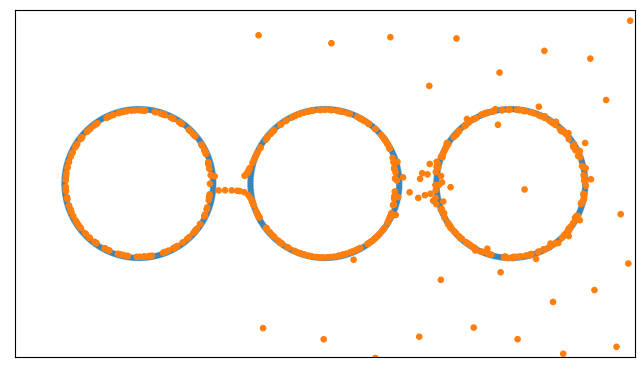}
\end{subfigure}
\end{minipage}

\begin{minipage}{\linewidth}
\rotatebox{90}{
\begin{minipage}{0.125\linewidth}
\centering
$\sigma^2 = 0.1$
\end{minipage}
}
\begin{subfigure}[t]{.23\textwidth}
\includegraphics[width=\linewidth]{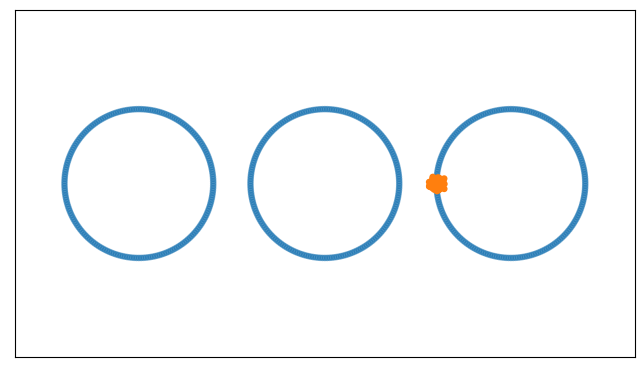}
\caption*{T=0}
\end{subfigure}
\hfill
\begin{subfigure}[t]{.23\textwidth}
\includegraphics[width=\linewidth]{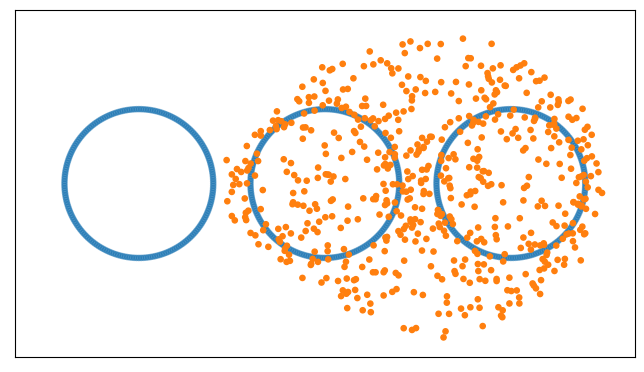}
\caption*{T=2}
\end{subfigure}
\hfill
\begin{subfigure}[t]{.23\textwidth}
\includegraphics[width=\linewidth]{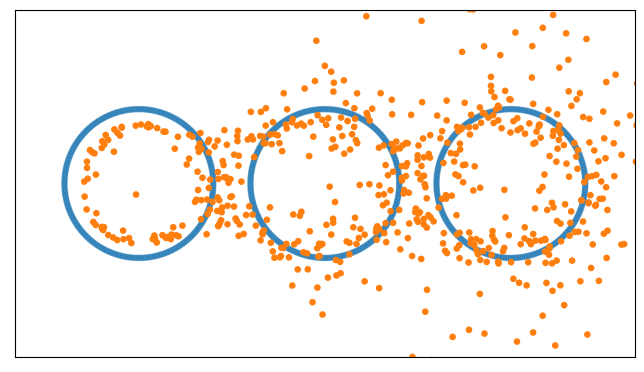}
\caption*{T=10}
\end{subfigure}
\hfill
\begin{subfigure}[t]{.23\textwidth}
\includegraphics[width=\linewidth]{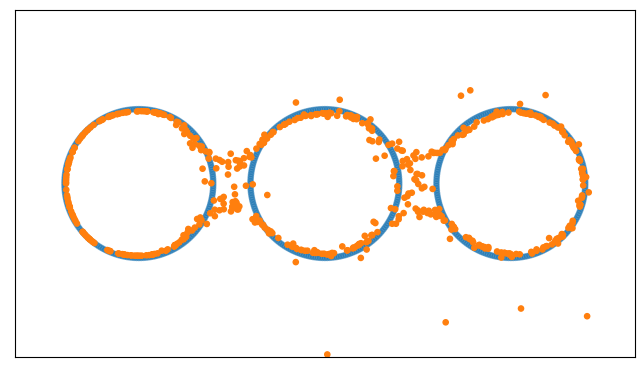}
\caption*{T=100}
\end{subfigure}
\end{minipage}

\centering
Riesz kernel
\begin{minipage}{\linewidth}
\rotatebox{90}{
\begin{minipage}{0.125\linewidth}
\centering
$r=0.5$
\end{minipage}
}
\begin{subfigure}[t]{.23\textwidth}
\includegraphics[width=\linewidth]{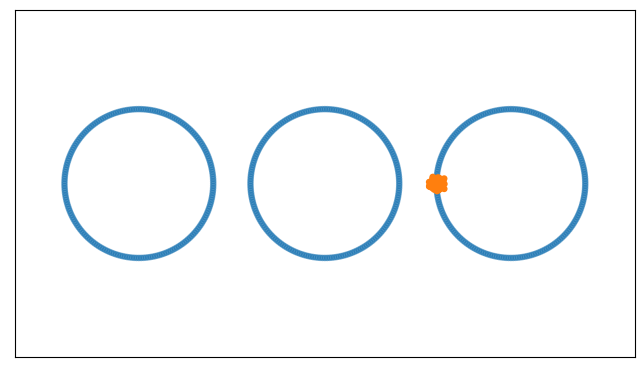}
\end{subfigure}
\hfill
\begin{subfigure}[t]{.23\textwidth}
\includegraphics[width=\linewidth]{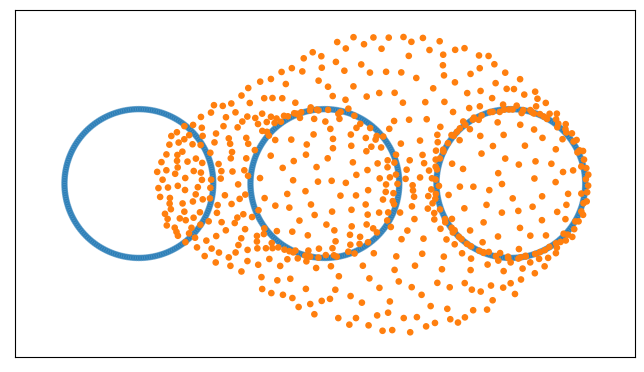}
\end{subfigure}
\hfill
\begin{subfigure}[t]{.23\textwidth}
\includegraphics[width=\linewidth]{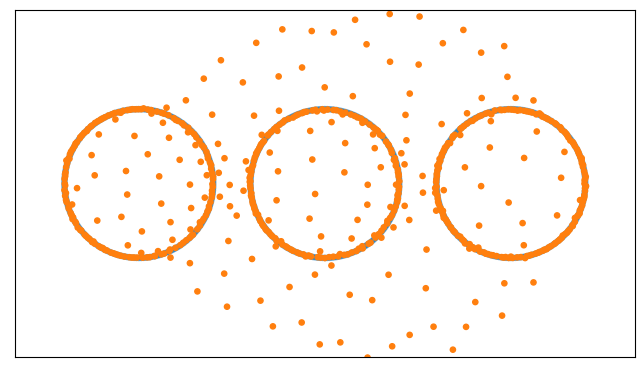}
\end{subfigure}
\hfill
\begin{subfigure}[t]{.23\textwidth}
\includegraphics[width=\linewidth]{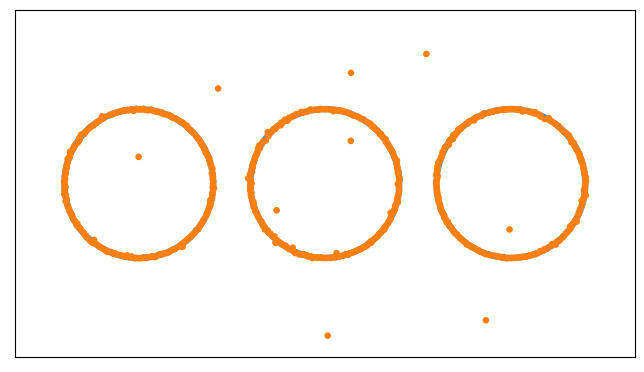}
\end{subfigure}
\end{minipage}

\begin{minipage}{\linewidth}
\rotatebox{90}{
\begin{minipage}{0.125\linewidth}
\centering
$r=1.0$
\end{minipage}
}
\begin{subfigure}[t]{.23\textwidth}
\includegraphics[width=\linewidth]{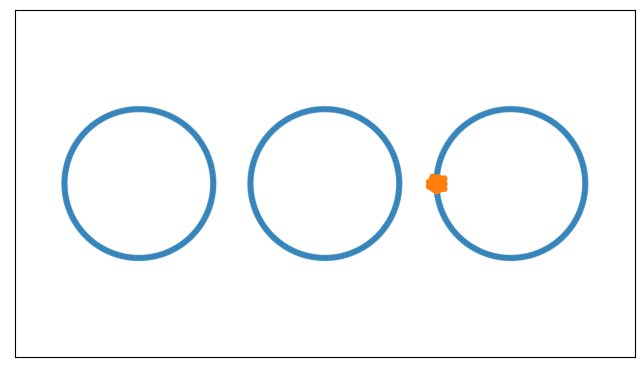}
\end{subfigure}
\hfill
\begin{subfigure}[t]{.23\textwidth}
\includegraphics[width=\linewidth]{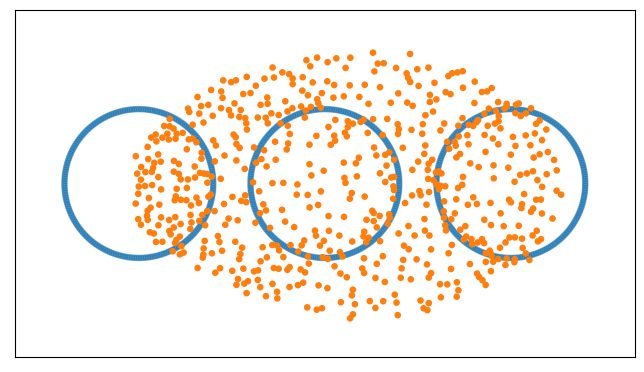}
\end{subfigure}
\hfill
\begin{subfigure}[t]{.23\textwidth}
\includegraphics[width=\linewidth]{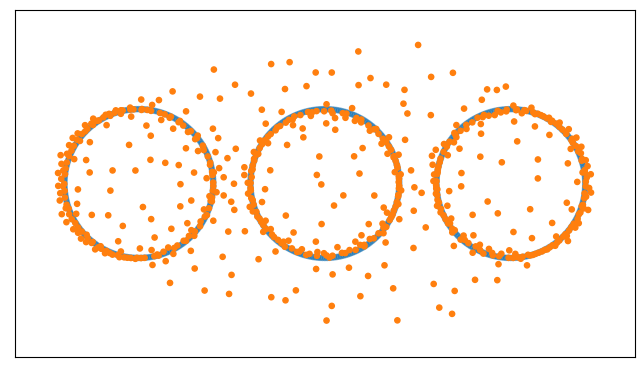}
\end{subfigure}
\hfill
\begin{subfigure}[t]{.23\textwidth}
\includegraphics[width=\linewidth]{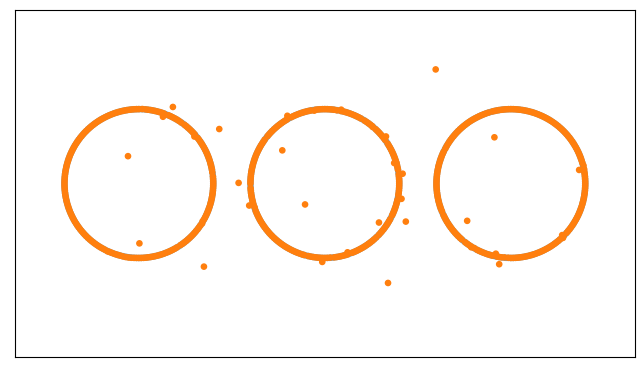}
\end{subfigure}
\end{minipage}

\begin{minipage}{\linewidth}
\rotatebox{90}{
\begin{minipage}{0.125\linewidth}
\centering
$r=1.5$
\end{minipage}
}
\begin{subfigure}[t]{.23\textwidth}
\includegraphics[width=\linewidth]{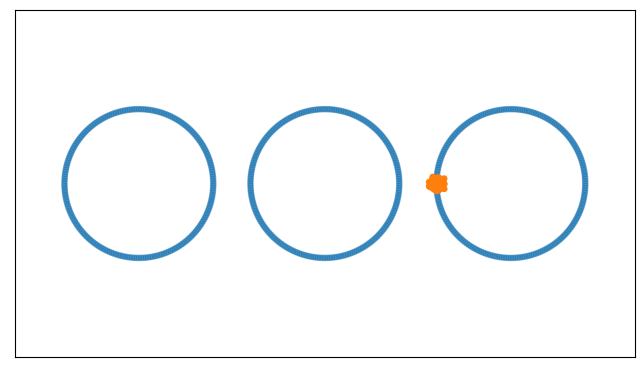}
\caption*{T=0}
\end{subfigure}
\hfill
\begin{subfigure}[t]{.23\textwidth}
\includegraphics[width=\linewidth]{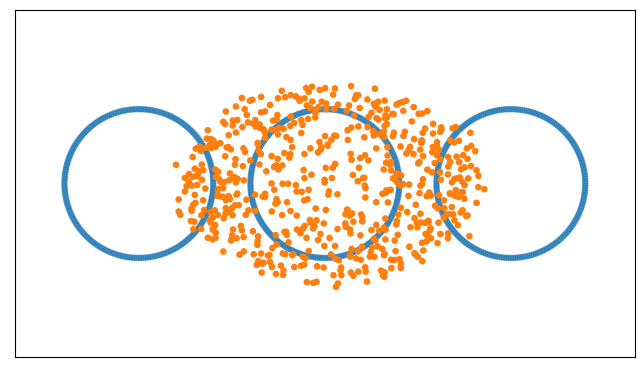}
\caption*{T=2}
\end{subfigure}
\hfill
\begin{subfigure}[t]{.23\textwidth}
\includegraphics[width=\linewidth]{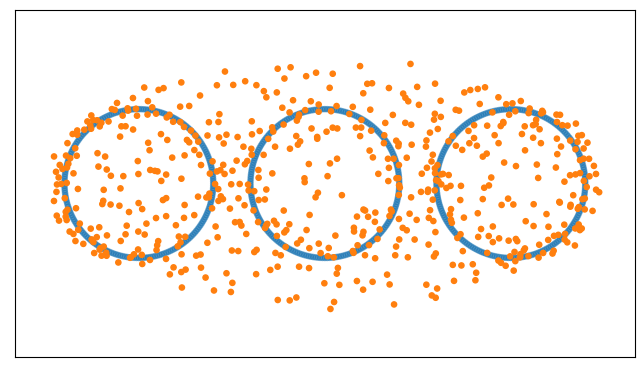}
\caption*{T=10}
\end{subfigure}
\hfill
\begin{subfigure}[t]{.23\textwidth}
\includegraphics[width=\linewidth]{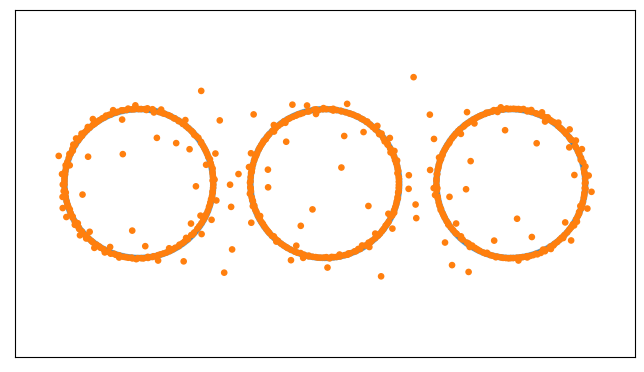}
\caption*{T=100}
\end{subfigure}
\end{minipage}
\caption{Comparison of the MMD flow with Laplacian kernel (top) and Riesz kernel (bottom) for different hyperparameters.} \label{fig:different_kernels2}
\end{figure*}

\section{Training Algorithm of the Generative Sliced MMD Flow} \label{app:training_algorithm}

In Algorithm~\ref{alg:training_gen_MMD_flows} we state the detailed training algorithm of our proposed method.

\begin{algorithm}[t]
\begin{algorithmic}
\State \textbf{Input:} Independent initial samples $x_1^{(0)},...,x_N^{(0)}$ from $\mu_0$, momentum parameters $m_l\in[0,1)$ for $l=1,...,L$.
\State Initialize $(v_1,...,v_N)=0$.
\For{$l=1,...,L$}
\State - Set $(\tilde x_1^{(0)},...,\tilde x_N^{(0)})=(x_1^{(l-1)},...,x_N^{(l-1)})$.
\State - Simulate $T_l$ steps of the (momentum) MMD flow:
\For{$t=1,...,T_l$}
\State - Update $v$ by
\begin{align*}
(v_1,...,v_N)\leftarrow \nabla F_d(\tilde x_1^{(t-1)},...,\tilde x_N^{(t-1)}|y_1,...,y_M)+m_l (v_1,...,v_N)
\end{align*}
\State - Update the flow samples:
\begin{align*}
(\tilde x_1^{(t)},...,\tilde x_N^{(t)})=(\tilde x_1^{(t-1)},...,\tilde x_N^{(t-1)})-\tau N\ (v_1,...,v_N)
\end{align*}
\EndFor
\State - Train $\Phi_l$ such that $\tilde x^{(T_l)}\approx \tilde x_i^{(0)}-\Phi_l(\tilde x_i^{(0)}) $ by minimizing the loss
\begin{align*}
\mathcal L(\theta_l)=\frac1N\sum_{i=1}^N \|\Phi_l(\tilde x_i^{(0)})-(\tilde x_i^{(0)}-\tilde x_i^{(T_l)})\|^2.
\end{align*}
\State - Set $(x_1^{(l)},...,x_N^{(l)})=(x_1^{(l-1)},...,x_N^{(l-1)})-(\Phi_l(x_1^{(l-1)}),...,\Phi_l(x_N^{(l-1)}))$.
\EndFor
\end{algorithmic}
\caption{Training of generative MMD flows}
\label{alg:training_gen_MMD_flows}
\end{algorithm}

\section{Ablation Study}
We consider the FID for different number of networks and different number of projections. We run the same experiment as in Section~\ref{sec:5} on MNIST. Here we choose a different number of projections $P$ between $10$ and $1000$. In Figure~\ref{fig:fid_plot} we illustrate the progress of the FID value for an increasing number of networks and a different number of projections. Obviously, the gradient of the MMD functional is not well-approximated by just using $P=10$ or $P=100$ projections and thus the MMD flow does not converges. Once the gradient of the functional is well-approximated, a higher number of projections leads only to a small improvement, see the difference between $P=500$ and $P=1000$.
\begin{figure*}[t!]
\centering
\begin{subfigure}[t]{.5\textwidth}
\includegraphics[width=\linewidth]{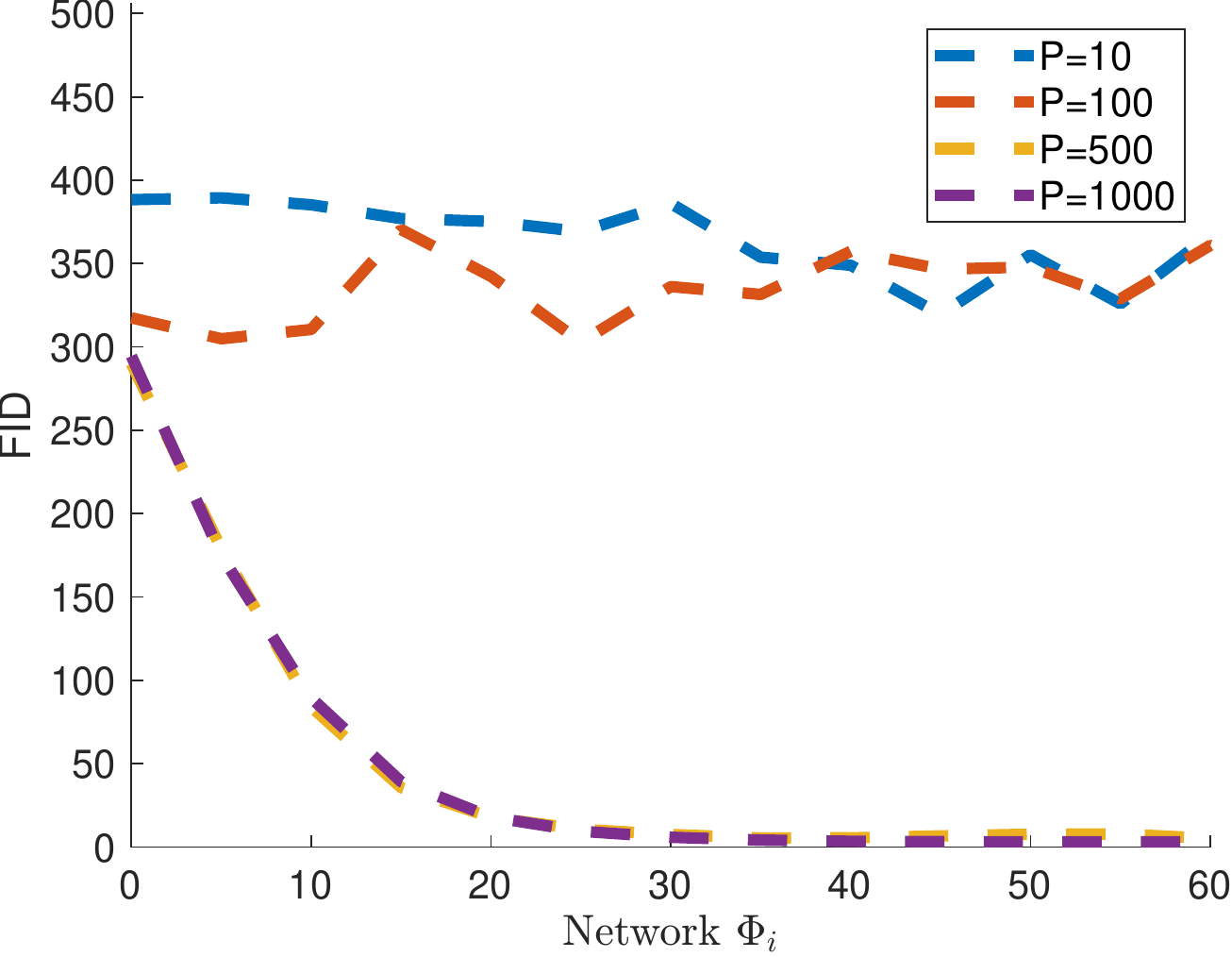}
\end{subfigure}
\caption{Illustration of the FID value of the Sliced MMD Flow on MNIST for different number of projections.} \label{fig:fid_plot}
\end{figure*}

\section{Implementation Details}\label{app:implementation}

The code is available online at \url{https://github.com/johertrich/sliced_MMD_flows}.

We use UNets $(\Phi)_{l=1}^L$\footnote{modified from \url{https://github.com/hojonathanho/diffusion/blob/master/diffusion_tf/models/unet.py}} with 3409633 trainable parameters for MNIST and FashionMNIST and 2064035 trainable paramters for CIFAR10. The networks are trained using Adam \citep{KB2015} with a learning rate of $0.001$. All flows are simulated with a step size $\tau=1$. We stop the training of our generative sliced MMD flow when the FID between the generated samples and some validation samples does not decrease twice. Then we take the network with the best FID value to the validation set. The validation samples are the last 10000 training samples from the corresponding dataset which were not used for training the generative sliced MMD flow. The training of the generative MMD flow takes between 1.5 and 3 days on a NVIDIA GeForce RTX 2060 Super GPU, depending on the current GPU load by other processes.
To avoid overfitting, we choose a relatively small number of optimization steps within the training of the networks $\Phi_l$, which corresponds to an early-stopping technique.

\textbf{MNIST.}
We draw the first $M=20000$ target samples from the MNIST training set and $N=20000$ initial samples uniformly from $[0,1]^d$. 
Then we simulate the momentum MMD flow using $P=1000$ projections for 32 steps and train the network for 2000 optimizer steps with a batch size of 100. 
After each training of the network, we increase the number of flow steps by $\min(2^{5+l},2048)$ up to a maximal number of 30000 steps, where $l$ is the iteration of the training procedure, see Algorithm~\ref{alg:training_gen_MMD_flows}.  We choose the momentum parameter $m=0.7$ and stop the whole training after $L=55$ networks.

\textbf{FashionMNIST.}
Here we draw the first $M=20000$ target samples from the FashionMNIST training set and $N=20000$ initial samples uniformly from $[0,1]^d$. Then we simulate the momentum MMD flow using $P=1000$ projections for 32 steps and train the network for 2000 optimizer steps with a batch size of 100. After each training of the network, we increase the number of flow steps by $\min(2^{5+l},2048)$ up to a maximal number of 50000 steps, where $l$ is the iteration of the training procedure. The momentum parameter $m$ is set to 0.8. We stop the whole training after $L=67$ networks.

\textbf{CIFAR.}
We draw the first $M=30000$ target samples from the CIFAR10 training set.
Here we consider a \textit{pyramidal schedule}, where the key idea is to run the particle flow on different resolutions, from low to high sequentially. First, we downsample the target samples by a factor 8 and draw $N=30000$ initial samples uniformly from $[0,1]^{\frac{d}{64}}$. 
Then we simulate the momentum MMD flow in dimension $d=48$ using $P=500$ projections for 32 steps and train the network for 5000 optimizer steps with a batch size of 100. After each training of the network, we increase the number of flow steps by $\min(2^{5+l},1024)$ up to a maximal number of 30000 steps, where $l$ is the iteration of the training procedure. The momentum parameter $m$ is increased after each network training by 0.01 up to 0.8, beginning with $m=0$ in the first flow step. 
We increase the resolution of the flow after $600000$ flow steps by a factor $2$ and add Gaussian noise on the particles in order to increase the intrinsic dimension of the images, such that the second resolution is of dimension $d=192$. Following here the same procedure as before for the second resolution and the third resolution of $d=768$, we change the projections in the final resolution of $d=3072$. Instead of using projections uniformly sampled from $\Sp^{d-1}$, we consider \textit{locally-connected} projections as in \citep{DLPYL2023,nguyen2022revisiting}. The idea is to extract patches of the images $\zb{x}^k$ at a random location in each step $k$ and instead consider the particle flow in the patch dimension. In order to apply these locally-connected projections at different resolutions, we also upsample the projections to different scales. Here we choose a patch size of $7 \times 7$ and consider the resolutions $7,14,21,28$. We stop the whole training after $L=86$ networks.

Note that herewith we introduced an inductive bias, since we do not uniformly sample from $[0,1]^d$, but empiricially this leads to a significant acceleration of the flow. A more comprehensive discussion can be found in \citep{DLPYL2023,nguyen2022revisiting}

\textbf{CelebA.}
We draw the first $M=20000$ target samples from the CelebA training set.
Again, we consider a pyramidal schedule as for CIFAR10, but we increased the number of flow steps by $\min(2^{5+l},8192)$ up to a maximal number of 100000 steps.
We increase the resolutions of the flow after $700000$, $900000$ and $700000$ flow steps by a factor $2$ and add Gaussian noise on the particles in order to increase the intrinsic dimension of the images. We also use locally-connected projections with a patch size of $7 \times 7$ and consider the resolutions $7,14,21,28$ for resolution $32 \times 32$ and $7,14,21,28,35,42,49$ for resolution $64 \times 64$. We stop the whole training after $L=71$ networks.

\section{Additional examples} \label{app:add_examples}

In Figure~\ref{fig:add_generated_samples} we show more generated samples of our method for MNIST, FashionMNIST and CIFAR10.

\begin{figure*}[t!]
\centering
\begin{subfigure}[t]{1\textwidth}
\includegraphics[width=\linewidth]{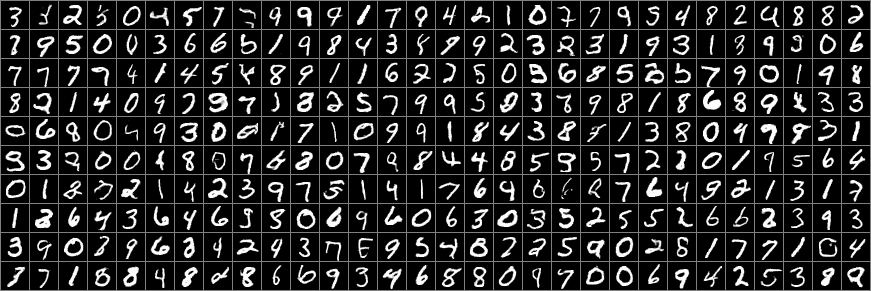}
\caption{MNIST}
\end{subfigure}

\begin{subfigure}[t]{1\textwidth}
\includegraphics[width=\linewidth]{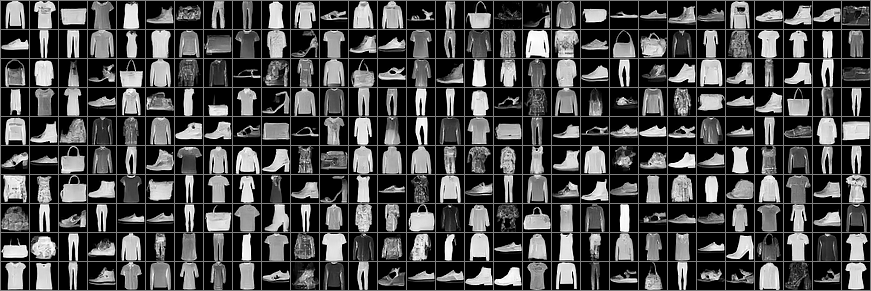}
\caption{FashionMNIST}
\end{subfigure}

\begin{subfigure}[t]{1\textwidth}
\includegraphics[width=\linewidth]{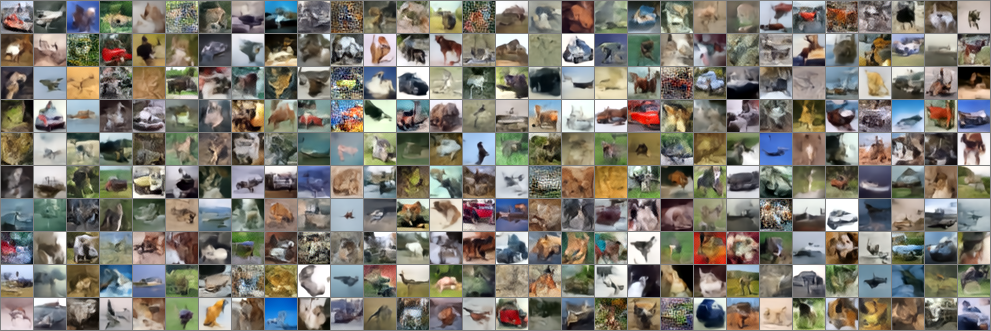}
\caption{CIFAR10}
\end{subfigure}

\begin{subfigure}[t]{1\textwidth}
\includegraphics[width=\linewidth]{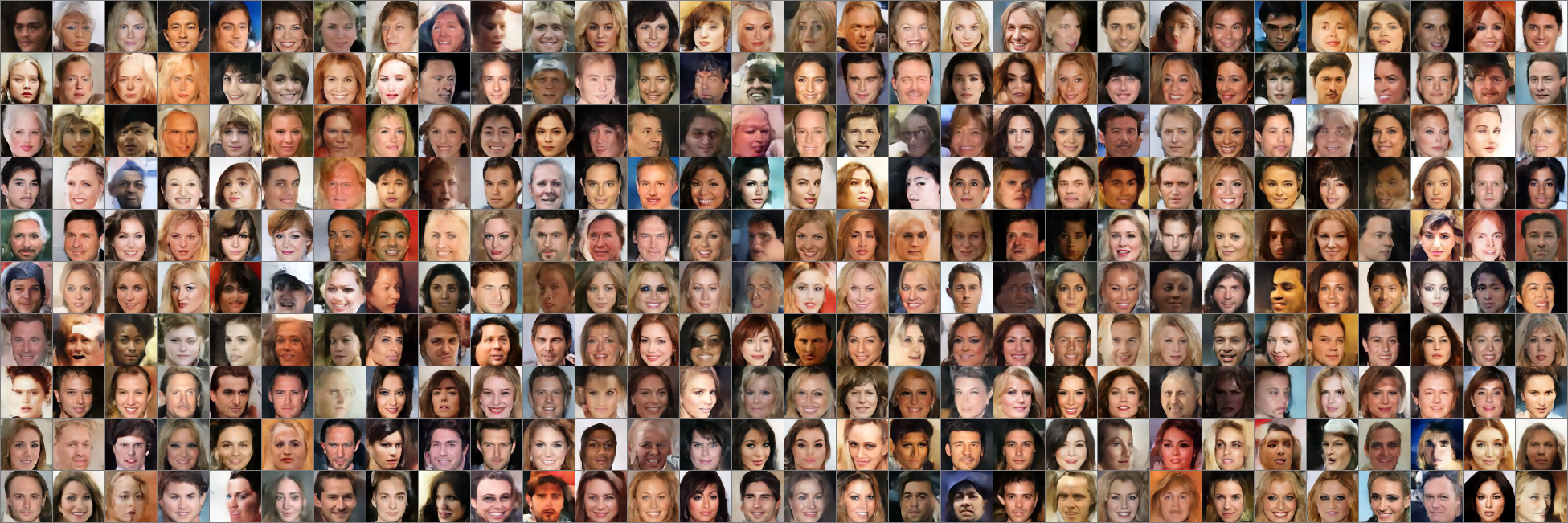}
\caption{CelebA}
\end{subfigure}
\caption{Additional generated samples (from top to bottom) of MNIST, FashionMNIST, CIFAR10 and CelebA.}
\label{fig:add_generated_samples}
\end{figure*}

In Figure~\ref{fig:diff_imgs}, we compare generated MNIST samples with the closest samples from the training set. We observe that they are significantly different. Hence, our method generates really new samples and is not just reproducing the samples from the training set. In contrast, in Figure~\ref{fig:diff_imgs_flow} we compare the particle flow samples with the closest samples from the training set. Obviously, the samples of the particle flow approximate exactly the training samples. This highlights the important role of the networks: We can interpolate between the training points in order to generalize the dataset.

\begin{figure*}[t!]
\centering
\begin{subfigure}[t]{\textwidth}
\includegraphics[width=\linewidth]{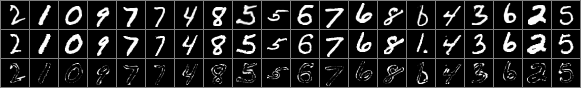}
\caption{MNIST}
\end{subfigure}

\begin{subfigure}[t]{\textwidth}
\includegraphics[width=\linewidth]{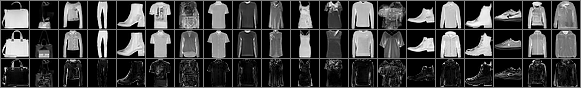}
\caption{FashionMNIST}
\end{subfigure}

\begin{subfigure}[t]{\textwidth}
\includegraphics[width=\linewidth]{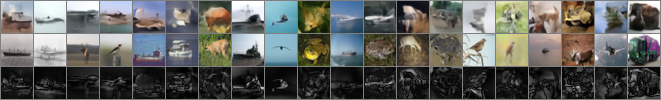}
\caption{CIFAR10}
\end{subfigure}

\begin{subfigure}[t]{\textwidth}
\includegraphics[width=\linewidth]{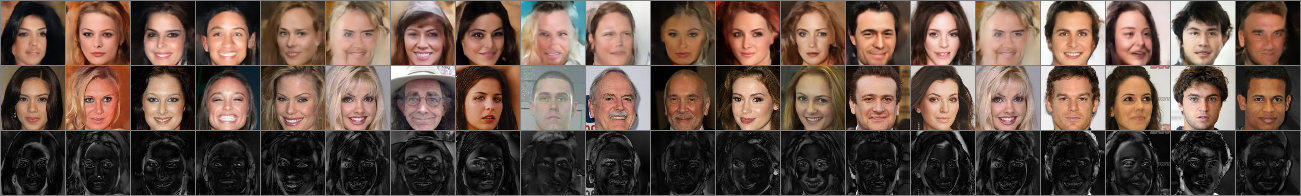}
\caption{CelebA}
\end{subfigure}
\caption{
Generated samples (top), $L_2$-closest samples from the training set (middle) and the pixelwise distance between them (bottom).
}
\label{fig:diff_imgs}
\end{figure*}

\begin{figure*}[t!]
\centering
\begin{subfigure}[t]{\textwidth}
\includegraphics[width=\linewidth]{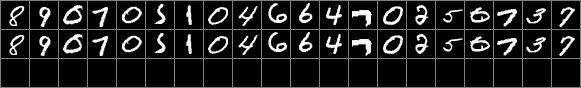}
\end{subfigure}
\caption{
Samples of the particle flow (top), closest samples from the training set (middle) and the pixelwise distance between them (bottom). The mean PSNR between these flow samples and the corresponding closest training images is 82.29. 
}
\label{fig:diff_imgs_flow}
\end{figure*}

\begin{table}[t]
\caption{FID scores of generated samples for training set and test set}  
\label{table:FID_train_test}
\centering
\scalebox{.71}{
\begin{tabular}[t]{c|ccccc} 
             & MNIST   & FashionMNIST & CIFAR10  & CelebA \\
\hline
training set &  2.7    &  10.6        &  53.0    &  32.7       \\
test set     &  3.1    &  11.3        &  54.8    &  32.1        \\
\end{tabular}}
\end{table}

\end{document}